\documentclass[twoside, accepted]{article}

\usepackage[utf8]{inputenc} 
\usepackage[T1]{fontenc}    
\usepackage{hyperref}       
\usepackage{url}            
\usepackage{booktabs}       
\usepackage{nicefrac}       
\usepackage[nopatch=footnote]{microtype}     
\usepackage{bbm}
\usepackage{graphicx}
\usepackage{color}
\usepackage[dvipsnames]{xcolor}
\usepackage[T1]{fontenc}
\usepackage{subcaption}
\usepackage{dsfont}
\usepackage{amsfonts}
\usepackage{graphicx,wrapfig}
\usepackage[T1]{fontenc}
\usepackage{amsmath}
\usepackage{mathtools, cuted}
\usepackage{amsthm}
\usepackage{amstext}
\usepackage{amssymb}
\usepackage{mathrsfs}
\usepackage{marginnote}
\usepackage{tkz-tab}
\usetikzlibrary{topaths,calc}
\usepackage{pgfplots}
\pgfplotsset{compat=1.18}
\usepackage{stackengine}
\usepackage{multicol}

\DeclareMathAlphabet{\mathbsf}{OT1}{cmss}{bx}{n}
\DeclareMathAlphabet{\mathssf}{OT1}{cmss}{m}{sl}
\DeclareSymbolFont{bsfletters}{OT1}{cmss}{bx}{n}  
\DeclareSymbolFont{ssfletters}{OT1}{cmss}{m}{n}

\usetikzlibrary{arrows,positioning, shapes}
\usepackage{float}

\usepackage{setspace}
\usepackage{hyperref}

\hypersetup{
    colorlinks,
    linkcolor={blue},
    citecolor={green!60!black},
    urlcolor={blue}
}

\def\eps{\varepsilon}

\def\D{{\mathcal D}}

\def\Q{{\mathcal Q}}

\def\X{{\mathcal X}}

\def\sK{{\mathsf K}}

\def\sM{{\mathsf M}}

\def\tv{{\mathsf {TV}}}
\def\kl{{\mathsf {KL}}}

\usepackage[font=small,labelsep=space]{caption}
\DeclareCaptionLabelSeparator{dot}{.~}
\newcounter{example}
\newenvironment{example}[1][]{\refstepcounter{example}\par\medskip
   \noindent \textit{Example~\theexample. #1} \rmfamily}{\medskip}
\newtheorem{definition}{Definition}
\newtheorem{theorem}{Theorem}

\newtheorem{corollary}{Corollary}

\newtheorem{proposition}{Proposition}
\newtheorem{lemma}{Lemma}

\usepgfplotslibrary{fillbetween}
\usepackage{times}
\usepackage{tikz}
\usepackage{amsmath, bm}
\usepackage{verbatim}
\usetikzlibrary{arrows,shapes}
\tikzstyle{RectObject}=[rectangle,fill=white,draw,line width=0.2mm]
\tikzstyle{line}=[draw]
\tikzstyle{arrow}=[draw, -latex]
\usetikzlibrary{decorations.pathmorphing}

\definecolor{DukeBlue}{HTML}{001A57}
\definecolor{DarkRed}{rgb}{0.75, 0.0, 0.0}
\definecolor{DarkGreen}{rgb}{0.0, 0.5, 0.0}

\usepackage[normalem]{ulem}

\allowdisplaybreaks

\usepackage{float}
\usepackage{stfloats}

\def \eps{\varepsilon}
\usepackage{cleveref}
\def\sM{{\mathsf M}}

\usepackage{algorithm}
\usepackage{algorithmic}

\newenvironment{sketchofproof}[1][Sketch of Proof]{\begin{proof}[#1]}{\end{proof}}

\def\r{{\frac{e^\eps q_{\min}}{e^\eps q_{\min} + 1 - q_{\min}}}}

\def\ro{{\frac{e^\eps q_{1}}{e^\eps q_{1} + 1 - q_{1}}}}

\def\ri{\frac{e^\eps q_{i}}{e^\eps q_{i} + 1 - q_{i}}}

\def\sKp{({\mathsf K}_{q,\eps})}
\def\sKo{{\mathsf K}_{q,\eps}}

\def\sKpu{({\mathsf K}_{q_{\text{u}},\eps})}
\def\sKou{{\mathsf K}_{q_{\text{u}},\eps}}

\def\df{{D_f}}

\def\optutil{{\Gamma_{f}(q, \eps)}}

\usepackage[utf8]{inputenc} 
\usepackage[T1]{fontenc}    
\usepackage{hyperref}       
\usepackage{url}            
\usepackage{booktabs}       
\usepackage{amsfonts}       
\usepackage{nicefrac}       
\usepackage{microtype}      
\usepackage{xcolor}         

\usepackage{aistats2025}
\usepackage[round]{natbib}

\begin{document}

%

%

\twocolumn[

\aistatstitle{Locally Private Sampling with Public Data }

\aistatsauthor{ Behnoosh Zamanlooy \And Mario Diaz  \And  Shahab Asoodeh }

\aistatsaddress{ McMaster University \\ Vector Institue \\ zamanlob@mcmaster.ca 
\And   IIMAS \\ Universidad Nacional Autónoma de México \\ mario.diaz@sigma.iimas.unam.mx 
\And McMaster University \\ Vector Institute \\ asoodeh@mcmaster.ca } ]


\begin{abstract}

Local differential privacy (LDP) is increasingly employed in privacy-preserving machine learning to protect user data before sharing it with an untrusted aggregator. Most LDP methods  assume that users possess only a single data record, which is a significant limitation since users often gather extensive datasets (e.g., images, text, time-series data) and frequently have access to public datasets. To address this limitation, we propose a locally private sampling framework that leverages both the private and public datasets of each user. Specifically, we assume each user has two distributions: $p$ and $q$ that represent their private dataset and the public dataset, respectively. The objective is to design a mechanism that generates a private sample approximating $p$ while simultaneously preserving  $q$. We frame this objective as a minimax optimization problem using $f$-divergence as the utility measure. We fully characterize the minimax optimal mechanisms for general $f$-divergences provided that $p$ and $q$ are discrete distributions. Remarkably, we demonstrate that this optimal mechanism is universal across all $f$-divergences. Experiments validate the effectiveness of our minimax optimal sampler compared to the state-of-the-art  private sampler.

\end{abstract}


\section{Introduction}\label{Section_intro}

Differential privacy (DP) has become the de facto standard for ensuring privacy in machine learning, and it has been widely adopted by major technology companies such as Google \citep{erlingsson2014rappor,Prochlo}, Microsoft \citep{ding2017collecting}, LinkedIn \citep{rogers2020linkedins},  and Meta \citep{yousefpour2022opacus}. Intuitively, an algorithm is considered differentially private if small changes to the input data---such as modifying a single entry---do not significantly affect the algorithm’s output. This traditional model assumes that a trusted curator has full access to the dataset, which can be a limiting factor in many practical scenarios.

To relax this assumption, local differential privacy (LDP) was introduced by \citet{Shiva_subsampling}, allowing users to randomize their data on their own devices before sharing it with an \textit{untrusted} curator. This model, with considerably weaker assumption in the trust model, has made LDP particularly appealing to both users and tech companies --- the first large-scale deployments of DP was in the local model \citep{Apple_Privacy, erlingsson2014rappor}.

However, a significant limitation of traditional LDP approaches is the assumption that each user only possesses a single data record. In reality, users typically have multiple data records across different modalities, such as photos, text messages, or even whole datasets. A line of research, called user-level privacy, addresses this by assuming that the users possess datasets of the same size generated from an underlying distribution \citep{Ulevel2, Ulevel3, Ulevel4, Ulevel1}. Although this line of work is more realistic, the assumption that  all users must have a dataset of the same size introduces practical limitations.
To mitigate this shortcoming, \citet{husain2020local} proposed treating each user’s local data as a probability distribution. They then developed a differentially private \textit{sampling} algorithm that takes a user's probability distribution as input and releases a single sample that closely approximates the distribution. This method demonstrated better performance compared to ad-hoc private sampling techniques such as private kernel density estimation \citep{Private_KDE} and private GAN \citep{DP_GAN}. Nevertheless, it still has a major limitation: it relies heavily on an arbitrary reference distribution. As a result, the performance of the algorithm is highly sensitive to the choice of reference distribution, making it less reliable in practice.
To address these limitations, \citep{park2024exactly} recently developed a minimax formulation for locally private sampling. They fully characterized the minimax risks in terms of all general $f$-divergences in both discrete and continuous settings and identified the optimal samplers that achieve these risks. 
Their empirical results demonstrated that their samplers significantly outperform the approach of \citet{husain2020local} for any reference distribution.

Building on this framework, we extend the work of \citet{park2024exactly} to incorporate publicly available data that do not require privacy protection. Our goal is to design optimal  samplers that generate private samples while preserving the integrity of public data. This ensures that when the user's data distribution closely aligns with the public data (viewed as a distribution), the resulting sampling distribution remains similarly aligned.

\paragraph{Main Contributions} More precisely, in this paper we make the following contributions. 
\begin{itemize}
	\item We introduce a locally private sampler with a \textit{public prior} within a privacy-utility tradeoff framework. Let $p$ and $q$ represent the user's data and public prior, respectively.  The privacy-utility tradeoff defines the ``optimal'' private sampler via an $\varepsilon$-LDP mechanism $\sK$, which minimizes the worst-case $f$-divergence between $p$ and the resulting sampling distribution. The mechanism $\sK$ must satisfy two key properties: (i) $\sK$ is $\eps$-LDP (see \Cref{prelim_mollifiers} for the definition),  and (ii) $\sK$ does not perturb the public prior, that is if $p = q$, then the resulting sampling distribution is $q$. More generally, if $q$ is ``close'' to $p$, then the resulting sampling distribution should also be close to $p$. 
\item For discrete distributions, we fully characterize the privacy-utility tradeoff in Theorem~\ref{thrm_optimal_utility} for all $f$-divergences. We then propose an algorithm (\Cref{alg}) that produces an \textit{optimal} locally private sampler for any discrete public prior $q$ (Theorem~\ref{thrm_optimal_mechanism}). Surprisingly, our optimal sampler is \textit{universally optimal} under any choice of $f$-divergence.  This optimal sampler is shown to have a simple closed form expression if $q$ is Bernoulli (Proposition~\ref{lemma_binary}).    

\item As a special case, we delineate that, the randomized response mechanism \citep{warner1965randomized} is optimal when the public prior is the uniform distribution.  

\item We demonstrate the considerably better 
performance of our approach through comprehensive benchmarks against the baseline method proposed by \citet{husain2020local} on synthetic and real-world datasets.


\end{itemize}



\subsection{Related Work}\label{section_related_work}

Local differential privacy (LDP) has been widely applied to various statistical problems, such as distribution estimation \citep{Asoodeh_JSAIT2024,LDP_DistributionEstimation, Disribution_estimation_hadamard, kairouz16_LDPEstimation, LDP_Fisher, Jayadev_Unified, Optimal_compressionLDP_Feldman, Optimal_compressionLDP_Kairouz, feldman2022private}. In all of these works, LDP mechanisms are typically applied to a single data record per user, an assumption that overlooks practical cases where users possess datasets consisting of multiple records across different modalities.

A closely related problem to private distribution estimation is the private sampling problem, which focuses on privately generating a single sample from the underlying distribution, rather than learning the distribution itself. \citet{DP_Sampling} studied private sampling under the total variation distance in the \textit{central} model for discrete distributions, and this was later extended to multi-dimensional Gaussian distributions by \citet{DP_Sampling_Gaussian}.


In the local model, \citet{husain2020local} proposed an efficient algorithm for privately sampling from a distribution $p$. Their approach involves constructing a small ball of a certain radius (referred to as a ``relative mollifier''), around a reference distribution $q$ and projecting $p$ onto this ball with respect to the $\kl$-divergence (see \Cref{prelim_mollifiers} for further details). Notably, the reference distribution $q$ can be interpreted as a public prior, as the algorithm preserves its invariance, a property that aligns with our approach.

\citet{LLM} extended the relative mollifier concept to study R\'enyi differential privacy guarantees in the training of large language models. Their approach demonstrated superior performance compared to standard DP stochastic gradient methods \citep{abadi2016deep} on large-scale datasets, highlighting the broader applicability of private sampling techniques.

More recently, \citet{park2024exactly} introduced a minimax formulation to determine optimal samplers under the assumption that $p$ is absolutely continuous with respect to a positive measure with its Radon-Nikodym derivative satisfying a certain regularity condition. They identified two families of optimal samplers: linear and non-linear. A linear sampler generates samples from a distribution obtained by applying an $\eps$-LDP Markov kernel to $p$, while a non-linear sampler constructs the sampling distribution by projecting $p$ onto a convex set specified by $\eps$.

Recently, there has been a growing interest in incorporating public data into private learning problems (e.g., \cite{LDP_Public2, LDP_Public1,amid2022public,Ben_David_Public,Gautam_Public,Alon_Public,yu2021do_Public,Public_Private,Wang_Zhou_Public,kairouz21a_Public,bassily20a_Public} to name a few). Although most of these studies focus on the central DP model, notable exceptions include \cite{LDP_Public2, LDP_Public1}, which explore the role of public data in private mean estimation and regression, respectively. 

Motivated by these developments, we extend the framework of \citet{park2024exactly} in the linear setting to incorporate a public prior and identify optimal \textit{linear} private samplers. Consequently, our objective reduces to determining the optimal $\eps$-LDP Markov kernel $\sK$ that preserves the public prior. To the best of our knowledge, our work is the first to address the problem of private sampling with the inclusion of public data in the local model.


\subsection{Notation}

$X$ and $Z$ are used to denote random variables. We represent a set by $\mathcal{X}$ and the probability simplex over $\mathcal{X}$ by $\Delta(\mathcal{X})$. The support size of a discrete distribution is denoted by $|\mathcal{X}| = n$. A Dirac distribution, $\delta_i$, is a distribution concentrated at its $i$-th component. We also define $\operatorname{U}[0,1]$ as the uniform distribution on $[0,1]$, and $\operatorname{Ber}(\alpha)$ as the Bernoulli distribution with success probability $\alpha$.

Throughout, we use $\sK: \Delta(\mathcal{X}) \to \Delta(\mathcal{X})$ to denote a Markov kernel. Given a distribution $p \in \Delta(\mathcal{X})$, we let $p\sK$ represent the output distribution, defined as $p\sK(j) = \sum_{i} p(i) \sK_{ij}$.






\section{Preliminaries}
In this section, we formally present the concepts necessary for our main results.

\subsection{$f$-divergences}
Let $f:(0,\infty)\to (-\infty, \infty] $ be a convex function with $f(1) = 0$. Let $p$ and $q$ be two finitely supported probability distributions $p, q \in \Delta(\X)$. If $p \ll q$, then the $f$-divergence  between $p$ and $q$ is defined as 
$$
    \df(p \| q)=\sum_{x \in \X} q(x) f\big(\frac{p(x)}{q(x)}\big),
$$
with the understanding that $f(0) = f(0+)$ and $0f(0/0) = 0$.
Two popular instances of $f$-divergences are:     
(1) $\kl$-divergence,  $\kl(p\|q)\coloneqq D_f(p \| q)$ for $f(t) = t\log t$, and (2)  $\tv$-distance, $\tv(p, q)\coloneqq D_f(p \| q)$ for $f(t)=\frac{1}{2}|t-1|$.


\subsection{Mollifiers} \label{prelim_mollifiers}



\begin{definition}[Mollifiers, \cite{husain2020local}] 
Let $\mathcal{M} \subset \Delta(\mathcal{X})$ be a set of distributions  and $\varepsilon\geq 0$. We say $\mathcal{M}$ is an $\varepsilon$-mollifier if
\begin{equation*}\label{eq_mollifiers}
q(x) \leq e^\eps q^{\prime}(x),    
\end{equation*}  
for all $q, q' \in \mathcal{M}$ and all $x \in \mathcal{X}$.
\end{definition}

One particular construction of mollifiers is \textit{relative} mollifiers defined next.
\begin{definition}[Relative Mollifiers, \cite{husain2020local}]\label{def_RM}
Given a reference distribution $q\in \Delta(\mathcal X)$,  the \textit{relative} $\eps$-mollifier $\mathcal{M}_{\eps, q}$ is defined as
\begin{equation*}
   \mathcal{M}_{\eps, q} \coloneqq  
   \left\{ \tilde{q} \in \Delta(\X): \sup_{x \in \mathcal{X}} \max \left\{\frac{q(x)}{\tilde{q}(x)}, \frac{\tilde{q}(x)}{q(x)}\right\} \leq e^{\varepsilon / 2} \right\}.
\end{equation*}   
\end{definition}
It can be verified that any pair of distributions $p$ and $p'$ in $\mathcal{M}_{\eps, q}$ satisfy $p(x)\leq e^\eps p'(x)$ for all $x\in \mathcal X$, and thus $\mathcal{M}_{\eps, q}$ is an $\eps$-mollifier.

\subsection{Locally Private Samplers}\label{section_ldp_samplers}


\begin{definition}[Private Samplers \cite{husain2020local}] An $\eps$-locally private sampler is a randomized mapping $\mathcal{A}: \Delta(\X) \rightarrow \X$ such that for any $x \in \X$ and any two distributions $p, p^{\prime} \in \Delta(\mathcal{X})$ we have
$$
\frac{\operatorname{Pr}[\mathcal{A}(p)=x]}{\operatorname{Pr}\left[ \mathcal{A}(p^{\prime}) =x\right]} \leq e^\eps.
$$
\end{definition}
This definition closely mirrors the standard formulation of LDP \cite{Shiva_subsampling}, with the key difference that the input to the algorithm is a distribution, rather than a single data point. 
As highlighted by \citet{husain2020local}, one approach to ensure that a sampler is $\eps$-locally private is to guarantee that the distribution of  $\mathcal A(p)$ lies within an $\eps$-mollifier for any $p\in \Delta(\mathcal X)$. This observation motivates the following definition.


\begin{definition}
 A Markov kernel $\sK: \Delta(\X) \mapsto \Delta(\X)$ is said to be an $\eps$-LDP mechanism if
 \begin{equation}\label{eq_eps_LDP}
     \sup_{p,p' \in \Delta(\mathcal X) } \sup_{x \in \X}~ \frac{p \sK(x)}{p' \sK(x)} \leq e^\eps.
 \end{equation}
 We let $\Q_\eps$ denote the collection of all $\eps$-LDP mechanisms.
 \end{definition}
Notice that output distributions of any $\eps$-LDP $\sK$ forms an $\eps$-mollifier. Therefore, $p\sK$ for an $\eps$-LDP $\sK$ can be thought of as the sampling distribution of an $\eps$-private sampler. More precisely, one can construct a private sampler $\mathcal A$ as follows: $\mathcal A(p) = X$ where $X\sim p\sK$ for some $\eps$-LDP mechanism $\sK$. This observation allows us to reduce the problem of finding the optimal private sampler to identifying the optimal private mechanism $\sK$. In the subsequent sections, we follow this approach to design a minimax optimal sampler for the worst-case distribution.

It is important to note that  \cite{husain2020local} employed a fundamentally different methodology.
Their approach begins by selecting an arbitrary reference distribution $q\in \Delta(\mathcal X)$ and constructing $\mathcal M_{\eps, q}$, the relative $\eps$-mollifier around $q$. They then project $p$ onto $\mathcal M_{\eps, q}$ (using  $\kl$-divergence), and then sample from the projection.
Specifically, their sampling distribution $\hat p$ is obtained by solving   
\begin{equation}\label{eq_information_projection}
\kl(p \| \hat{p}) =  \inf_{p' \in \mathcal{M}_{\eps, q}} \kl(p \| p').   
\end{equation}

They use Karush-Kuhn-Tucker optimality conditions to show that solving \eqref{eq_information_projection} admits the following closed-form expression:
\begin{equation}\label{eq_information_projection_solution}
\hat{p}(x)=\min \left\{\max \left\{\frac{q(x)}{e^{\varepsilon / 2}}, \frac{p(x)}{C}\right\}, e^{\varepsilon / 2} q(x)\right\} ,   
\end{equation}
where  $C>0$ is a normalizing constant, depending on $p$, $q$, and $\varepsilon$, ensuring that $\hat{p}$ sums to one.

\section{Main Results: Minimax Optimal Mechanisms} \label{section_main_results}

As highlighted earlier, our goal is to design a minimax optimal private sampler by constructing an optimal $\eps$-LDP mechanism $\sK$. Specifically, we aim to identify a mechanism $\sK$ that guarantees a minimal gap between any distribution $p$ and its corresponding sampling distribution $p\sK$.

In many practical settings, publicly available information, such as demographic data or aggregate statistics, can alleviate the utility loss caused by privacy-preserving methods (see Example~\ref{example} for details on how public prior can help).  
Consequently, instead of ignoring public data, we constrain our search for optimal mechanisms $\sK$ to those that preserve this publicly available data, represented by a distribution $q$, without perturbing it. Specifically, we focus on mechanisms that enforce $q \sK = q$, ensuring that public data is invariant under the mechanism while maintaining privacy guarantees for private data. As a toy example, note that if $p = q$ (meaning the user's data is not private) then we would expect the optimal sampling distribution to be $p$ itself, which follows from the invariance requirement $q \sK = q$. This can be generalized to the case where $p$ and $q$ are close: a simple application of the triangle inequality and the invariance property imply that $\tv(p, p\sK)\leq 2\tv(p, q)$.



We next turn to formulate the minimax optimal mechanism $\sK$ satisfying the invariance property. Adopting $f$-divergence as a measure of distance, 
we define 
\begin{equation} \label{eq_original}
\optutil \coloneqq \inf_{ \sK \in \Q_\eps \atop q \sK = q} \sup_{p \in \Delta(\X)} \df (p \| p \sK).
\end{equation}

The mechanism that solves the above optimization is the minimax optimal mechanism with respect to the choice of $f$-divergence. In the case of $\tv$-distance, such optimal mechanism has a practical interpretation. Let $X\sim p$ and $Z$ is the output of an $\eps$-locally private sampler. Optimal mechanism attaining $\Gamma_{\tv}(q,\eps)$ corresponds to the best sampler that minimizes the worst-case $\Pr(X\neq Z)$.  

The following example intends to elaborate on the role of public data in the minimax formulation \eqref{eq_original}.  
\begin{example}\label{example}
Consider an advertising company deciding whether to display an ad on website A or B. Suppose a user belongs to a demographic with a available public data indicating that 1\% of people click on ads from website A, and the rest on B. Let the private data from the user's device suggest a 5\% click rate on A. 
\begin{itemize}
    \item\textbf{Without public prior}: If no public prior is available, we assume the uniform distribution $q_{\text{u}}$ over websites A and B. Solving \eqref{eq_original} under this assumption (as in \Cref{theorem_uniform}) for $\eps = 2$ results in $\Gamma_{\tv}(q_{\text{u}},2)=0.1$.
    \item\textbf{With public prior}: Incorporating the above public prior $q$ in \eqref{eq_original}, the optimal mechanism  (as detailed in \Cref{lemma_binary}) for $\eps = 2$ yields  $\Gamma_{\tv}(q, 2)=0.03$.
\end{itemize}
In this scenario, using public prior information enables the company to significantly reduce uncertainty in its targeting strategy, leading to better ad placement decisions compared to when the public information is not incorporated.   
\end{example}



In the next section, we fully characterize $\optutil$ and determine the optimal mechanism. 
Remarkably, the optimal mechanism does not depend on the choice of $f$-divergence.

  


\subsection{Optimal Utility}

Our first main result is an exact expression for the optimal utility $\optutil$.

\begin{theorem}[Optimal Utility for $f$-divergences]\label{thrm_optimal_utility}
Let $q$ be a finitely supported distribution and $ \displaystyle q_{\min} \coloneqq \min_{x \in \X} q(x)$. Then, 
\begin{equation} \label{eq_optimal_utility}
\begin{aligned}
    \optutil 
    &= \frac{1-q_{\min}}{e^\eps q_{\min} + 1 - q_{\min}} f(0) \\
    &+ \frac{e^\eps q_{\min}}{e^\eps q_{\min} + 1 - q_{\min}} f\left(\frac{e^\eps q_{\min} + 1 - q_{\min}}{e^\eps q_{\min}}\right).
\end{aligned}
\end{equation}

\end{theorem}

The characterization of optimal utility in this theorem helps to identify the smallest achievable gap between the worst-case distribution and its corresponding sampling distribution. Interestingly, it depends on the public prior only through its smallest element, $q_{\min}$.



Next, we instantiate the optimal utility presented in \eqref{eq_optimal_utility} for the total variation distance and we will use as a guide for designing an optimal mechanism and sampling algorithm in the subsequent section. 


\begin{corollary}[Optimal Utility for $\tv$-distance]\label{cor_optimal_utility_tv_kl}
Let $q$ be a finitely supported distribution and $ \displaystyle q_{\min} \coloneqq \min_{x \in X} q(x)$. Then,
\begin{align}
    \Gamma_{\tv}(q, \eps)
    &= \frac{1 - q_{\min}}{e^\eps q_{\min} + 1 - q_{\min}}. \label{eq_optimal_utility_tv}
\end{align}
\end{corollary}

 





\subsection{Optimal Mechanisms}

We first derive a closed-form expression for the optimal mechanism in the binary case. This step is crucial in our recursive construction presented later for the general optimal mechanism. 
\begin{proposition}[Optimal Binary Mechanism] \label{lemma_binary}
    Let $q \sim \operatorname{Ber(\alpha)}$ with $\alpha \leq \frac{1}{2}$. Then,
  the mechanism $\sKo$ that attains the optimal utility $\optutil$ is
\begin{equation}
    \sKo =\frac{1}{e^\eps \alpha + 1 - \alpha } \begin{bmatrix}
    e^\eps \alpha      & 1 - \alpha \\
    \alpha      & (e^\eps -1) \alpha + 1 - \alpha
\end{bmatrix} .
\end{equation}  
\end{proposition}
While this proposition provides a closed-form solution for the optimal binary mechanism, deriving the optimal mechanism for general discrete public priors is more complex. To address this, we present a recursive algorithm (\Cref{alg}) for general public priors with finite support.

We accomplish this goal by first reducing the search space in the optimization problem in $\optutil$, that is 
\begin{equation}\label{eq_sup_df}
\sup_{p \in \Delta(\X)} \df(p \| p\sK).    
\end{equation}
Next lemma delineates that it is sufficient to consider the extremal points of the simplex $\Delta(\X)$, i.e.\ the set of Dirac distributions on $\X$. This significant reduction in the search space allows us to precisely characterize \eqref{eq_sup_df}.

\begin{lemma}\label{lemma_equivalence}

Let $\X$ be a finite set. Given a Markov kernel $\sK$, define $\displaystyle \sK_{\min} \coloneqq \min_i \sK_{ii}$ and let $\delta_i$ be the Dirac distribution with its $i$th element equal to $1$. Then, we have that

\begin{align}
      \sup_{p \in \Delta(\X)} \df (p \| p \sK) 
     &= \max_{i} \df (\delta_i \| \delta_i \sK) \label{eq_minimax_quantification}  \\
     &=  f(0) (1 - \sK_{\min}) + \sK_{\min} f(\frac{1}{\sK_{\min}}). \nonumber  
\end{align}
 In particular, for the $\tv$-distance, we have that
\begin{equation}
 \sup_{p \in \Delta(\X)} \tv (p , p \sK) = \max_{i}~[1 - \sK_{ii}]. 
\end{equation}
\end{lemma}

Building on the equivalence relation in \Cref{lemma_equivalence} and the optimal utility derived in \Cref{thrm_optimal_utility}, we introduce \Cref{alg} and explain how it derives the optimal mechanism for general public priors. \Cref{fig_recursion_main} illustrates a single step of the recursion used in the algorithm.

For clarity and notational convenience, we focus on characterizing the optimal mechanism for the $\tv$-distance, while emphasizing that the optimal mechanism is independent of the $f$-divergence. We break down \Cref{alg} into four distinct steps.

\textbf{Input and Output.} \Cref{alg} takes the \textit{increasingly sorted} pulic prior, $q$, and the privacy parameter, $\varepsilon$, as inputs and returns the optimal mechanism, $\sKo$, as the output.

\textbf{Step 1.} \Cref{lemma_equivalence} and the characterization of the optimal utility for the $\tv$-distance in \eqref{eq_optimal_utility_tv} guide our intuition for setting the first diagonal element of the optimal mechanism, i.e.,
\begin{equation}
    \Gamma_{\tv}(q, \eps) = \inf_{\substack{\sK \in \Q_\eps \\ q \sK = q}} \max_{i}~[1 - \sK_{ii}] = \inf_{\substack{\sK \in \Q_\eps \\ q \sK = q}} \left( 1 - \min_{i} \sK_{ii} \right).    
\end{equation}

Hence, we need the smallest diagonal element to be
\begin{equation}\label{step1_equivalence}
\begin{aligned}
    \min_{i}~\sKp_{ii} &= 1 - \frac{1 - q_{\min}}{e^\eps q_{\min} + 1 - q_{\min}} \\
                       &= \frac{e^\eps q_{\min}}{e^\eps q_{\min} + 1 - q_{\min}}.    
\end{aligned}
\end{equation}

Therefore, we set $\sKp_{11}$ to 
\begin{equation}\label{step1_diagonal_value}
    \ro = \r,
\end{equation}
noting that $\sKp_{11} \leq \sKp_{22} \leq \dots \leq \sKp_{nn}$ by the construction in \Cref{alg}. Hence $\sKp_{11}$ will be the smallest diagonal entry, justifying our choice in \eqref{step1_diagonal_value}. 

\textbf{Step 2.} 
In this step, we assign the remaining entries in the first column as $\frac{q_{\min}}{e^\eps q_{\min} + 1 - q_{\min}}$ to comply with the $\eps$-LDP constraints for this column.

\textbf{Step 3.} In this step, we populate the block sub-matrix with rows and columns $2$ to $n$, denoted as $\sKp_{ij}$ for $i,j \geq 2$. To do this, we invoke a recursive call of \Cref{alg} to compute $\sK_{\bar{q}, \eps}$, where $\bar{q}$ is the normalized distribution using $q_2, \dots, q_n$, i.e., $\bar{q} = \frac{1}{\sum_{i=2}^n {q_i}} [q_2, \dots, q_n]$.

However, simply filling the block sub-matrix using $\sK_{\bar{q}, \eps}$ violates the requirement that an $\eps$-LDP mechanism should be row-stochastic. This is corrected by multiplying all the elements of $\sK_{\bar{q}, \eps}$ by $m = 1 - \frac{q_{\min}}{e^\eps q_{\min} + 1 - q_{\min}}$. We then use $m \sK_{\bar{q}, \eps}$ to fill in the block sub-matrix.

\textbf{Step 4.} We now fill the remaining elements in the first row. Since the values for the block sub-matrix are determined and $\sKo$ must satisfy the $q \sKo = q$ constraint, we deduce that there is only one choice for the remaining elements of the first row, namely $\sKp_{1j} = \frac{q_j}{e^\eps q_{\min} + 1 - q_{\min}}$ for $j \geq 2.$

\begin{figure}[!htb]
    \centering    
    \resizebox{\columnwidth}{!}{
    \(
    \sK_{q, \eps} = 
    \begin{bmatrix}
      \r & \begin{matrix} \frac{ q_2}{e^\eps q_{\min} + 1 - q_{\min}} & \dots & \frac{ q_n}{e^\eps q_{\min} + 1 - q_{\min}} \end{matrix} \\
      \begin{matrix}  \frac{ q_{\min}}{e^\eps q_{\min} + 1 - q_{\min}}\\ \vdots \\  \frac{ q_{\min}}{e^\eps q_{\min} + 1 - q_{\min}} \end{matrix}  &
      \begin{bmatrix}
        \hspace*{-\arraycolsep}
        \phantom{e^\eps q_{\min} + 1 } & \phantom{e^\eps q_{\min} + 1} & \phantom{e^\eps q_{\min} + 1}
        \hspace*{-\arraycolsep}
        \\
        & \raisebox{-0.2\height}[0pt][0pt]{\large$m \sK_{\bar{q}, \eps}$} & \\
        & &
      \end{bmatrix}
    \end{bmatrix}
    \)
    }
    \caption{Visualization of a single recursion step in \Cref{alg}.}
    \label{fig_recursion_main}
\end{figure}

\begin{algorithm}
\caption{\text{$\sK$\_with\_prior}: Algorithm to compute the optimal mechanism $\sKo$ for \Cref{eq_original}.}
\begin{algorithmic}[1]
\REQUIRE $q$ - Increasingly Sorted Public Prior, $\eps$ - Privacy Parameter
\ENSURE $\eps$-LDP mechanism $\sKo$ with $q \sKo = q$
\STATE $n \gets \text{length of } q$ 
\IF{$n=2$}
    \RETURN The optimal binary mechanism from \Cref{lemma_binary} \label{alg_binary}
\ELSE
    \STATE $\sKo \gets \text{zeros}(n,n)$
    \STATE $d \gets (e^{\eps} \cdot q_{\min} + 1 - q_{\min})$ \label{alg_denominator}
    \STATE $\sKp_{11} \gets \frac{ e^{\eps} \cdot q_{\min}}{d} $ \label{alg_first_column_start}
    \FOR{ $i \gets 2$ \TO $n$}
        \STATE  $\sKp_{i1} \gets \frac{q_{\min}}{d}$ \label{alg_column1}
    \ENDFOR \label{alg_first_column_finish}
    \STATE $\bar{q} \gets \frac{1}{\sum_{i=2}^n {q_i}} [q_2, \dots, q_n]$ \label{alg_q_bar}
    \STATE $\sK_{\bar{q}, \eps} \gets \text{$\sK$\_with\_prior}(\bar{q}, \eps)$ \label{alg_K_q_bar}
    \STATE $m \gets 1 - \frac{q_{\min}}{d}$
    \STATE $\sKo[2:(n), 2:(n)] \gets m \cdot \sK_{\bar{q}, \eps}$
    \FOR{ $i \gets 2$ \TO $n$}
        \STATE $\sKp_{1i} \gets \frac{q_i}{d}$ \label{alg_row1}
    \ENDFOR
    \RETURN $\sKo$
\ENDIF
\end{algorithmic}
\label{alg}
\end{algorithm}

Having justified the structure of the optimal mechanism $\sK_{q, \eps}$ obtained by Algorithm \ref{alg}, it is crucial to rigorously prove that it is $\eps$-LDP, it keeps $q$ invariant, and it achieves the optimal utility $\optutil$. The following theorem demonstrates that $\sK_{q, \eps}$ satisfies all these three properties.  

\begin{theorem}[Optimal Mechanisms]
\label{thrm_optimal_mechanism}
    Let $q$ be a finitely supported distribution  and $\displaystyle q_{\text{min}} \coloneqq \min_{x \in \X} q(x)$. Then $\sK_{q, \eps}$ the output of \Cref{alg} is $\eps$-LDP with $q \sK_{q,\eps} = q$ and it attains the optimal utility  in \Cref{thrm_optimal_utility}, i.e., 

\begin{equation}
 \begin{aligned}
    &\sup_{p \in \Delta(\X)} \df(p \| p \sK_{q, \eps})  
    = \frac{1-q_{\min}}{e^\eps q_{\min} + 1 - q_{\min}} f(0) \\
    &\quad + \frac{e^\eps q_{\min}}{e^\eps q_{\min} + 1 - q_{\min}} f\left(\frac{e^\eps q_{\min} + 1 - q_{\min}}{e^\eps q_{\min}}\right). 
    \end{aligned}   
\end{equation}

\end{theorem}

 \Cref{thrm_optimal_mechanism} also confirms that the choice of $f$-divergence for measuring distribution dissimilarity does not alter the optimal mechanism.

It is important to note that our algorithm achieves a time complexity of $O(n^3)$. This is a significant improvement over the naive linear programming approach for solving \eqref{eq_original} in the case of $\tv$-distance. The naive method has a complexity of $O\left(n^5 \log\left(\frac{n}{\delta}\right)\right)$ \citep{LPtime}, where $n$ is the support size of $p$ and $q$, and $\delta$ is the desired accuracy. With an $O(n^3)$ complexity, our method scales efficiently to much larger datasets.

Finally, we briefly examine the scenario where no public data is available. In this situation, a natural choice for a prior is the uniform distribution \( q_{\text{u}} \).

\begin{corollary}\label{theorem_uniform}
    Let $\X$ be a finite set and let $q_{\text{u}}$ represent the uniform distribution on $\X$. Then, the $n$-ary randomized response mechanism, $\sK_{\text{RR}}$ \citep{warner1965randomized}, is optimal in the sense:
    \begin{align}
        \Gamma_f(q_{\text{u}}, \eps)
        &= \sup_{p \in \Delta(\X)} \df(p \| p \sK_{\text{RR}}).
    \end{align}
\end{corollary}

This result confirms that $\sK_{\text{RR}}$ achieves the optimal utility when the prior is uniform and recovers the utility established by \cite{park2024exactly}, where no public prior was assumed.


Next, we provide a sketch of proof for the results presented in this section.

\section{Overview of Proofs}\label{section_overview}

Our proof strategy is divided into three steps. First, we establish a lower bound for $\optutil$. Next, we show that the optimal mechanism produced by \Cref{alg} is  $\eps$-LDP with $q \sKo = q$. Finally, we demonstrate that it achieves the optimal utility, $\optutil$.

Our argument for establishing the lower bound relies on \Cref{lemma_equivalence}, which demonstrates that the maximal distance between \( p \) and \( p \sK \), measured using any \( f \)-divergence and any kernel, occurs at Dirac distributions and is quantified by \eqref{eq_minimax_quantification}. This result, combined with the \(\eps\)-LDP constraints and the requirement that the public data remains invariant under the optimal mechanism, establishes the lower bound.

\begin{proposition} \label{proposition_converse}
Let $q$ be a finitely supported distribution and $ \displaystyle q_{\min} \coloneqq \min_{x \in \X} q(x)$. Then,  
\begin{equation}
\begin{aligned}
    &\optutil
    \geq \frac{1-q_{\min}}{e^\eps q_{\min} + 1 - q_{\min}} f(0) \\
    &+ \frac{e^\eps q_{\min}}{e^\eps q_{\min} + 1 - q_{\min}} f(\frac{e^\eps q_{\min} + 1 - q_{\min}}{e^\eps q_{\min}}).
\end{aligned}
\end{equation}
\end{proposition}

\vspace{-12pt}

\begin{sketchofproof}
    Let $\displaystyle \sK_{min} = \min_{i} \sK_{ii}$. By \Cref{lemma_equivalence} , we have that for a given mechanism $\sK$
    \begin{equation}\label{eq_f_div_equivalence}
        \sup_{p \in \Delta(\X)} D_{f} (p \| p \sK) 
        = \sK_{\min} f(\frac{1}{\sK_{\min}}) + f(0) (1 - \sK_{\min}).
    \end{equation}
    Additionally, using the $\eps$-LDP constraints and that $q \sK = q$, we can show that 
    $$
         \sK_{\min}\leq \frac{e^\eps q_{\min}}{e^\eps q_{\min} + 1 - q_{\min}}.
    $$
    Since $x \mapsto x f(\frac{1}{x}) + f(0) (1 - x)$ is a decreasing function for $x \in [0,1]$,
\begin{equation}
     \begin{aligned}
        &\sK_{\min} f(\frac{1}{\sK_{\min}}) + f(0) (1 - \sK_{\min})& \\
        &\geq
         \frac{1-q_{\min}}{e^\eps q_{\min} + 1 - q_{\min}} f(0) \\
         &+  \frac{e^\eps q_{\min}}{e^\eps q_{\min} + 1 - q_{\min}} f(\frac{e^\eps q_{\min} + 1 - q_{\min}}{e^\eps q_{\min}}).
    \end{aligned}   
\end{equation}

    By \eqref{eq_f_div_equivalence}, we have the result.
    \qedhere
\end{sketchofproof}


Now, we establish that the output of \Cref{alg}, $\sKo$, is a valid mechanism. 
\begin{proposition}
 \label{lemma_K_consistency}
The mechanism $\sKo$ calculated by \Cref{alg} is $\eps$-LDP  with $q \sKo = q$.   
\end{proposition}

Subsequently, we use the closed form characterization of the optimal binary mechanism presented in \Cref{lemma_binary} and induction to prove the optimality of the mechanism obtained from \Cref{alg}, $\sK_{q,\eps}$. To do this, we show that by the construction of $\sKo$, it achieves the lower bound presented in \Cref{proposition_converse}. This would complete the proof for \Cref{thrm_optimal_utility} and \Cref{thrm_optimal_mechanism}.

\begin{proposition}
\label{thrm_optimal_mechanism_part_2}
 Let $q$ be a finitely supported distribution  and  $\displaystyle q_{\text{min}} \coloneqq \min_{x \in X} q(x)$. Then, $\sK_{q, \eps}$, the mechanism obtained in \Cref{alg} attains the lower bound in \Cref{proposition_converse}, i.e.,
\begin{align}
\begin{split}
 &\sup_{p \in \Delta(\X)} \df(p \| p \sK_{q, \eps}) 
 = \frac{1-q_{\min}}{e^\eps q_{\min}  
 + 1 - q_{\min}} f(0) \\
 &~~~~ + \frac{e^\eps q_{\min}}{e^\eps q_{\min} + 1 - q_{\min}} f(\frac{e^\eps q_{\min} + 1 - q_{\min}}{e^\eps q_{\min}}).
 \end{split}
\end{align}
\end{proposition}

\begin{sketchofproof}
    Let $\sKo$ be the output of \Cref{alg}. Then we can show by construction of $\sKo$ and induction that $\sKp_{11} \leq \sKp_{22} \leq \dots \leq \sKp_{nn}$. Hence, 
    \begin{equation}\label{proposition4_eq}
       \displaystyle \min_{i}~\sKp_{ii} = \sKp_{11} = \frac{e^\eps q_{\min}}{e^\eps q_{\min} + 1 - q_{\min}}.
    \end{equation}
    Define $\displaystyle \sK_{\min} \coloneqq \min_{i}~\sKp_{ii}$. Then, by \Cref{lemma_equivalence}, we have that 
    \begin{align*}
    \centering
        \sup_{p \in \Delta(\X)} \D_f(p \| p \sKo) 
        &=  f(0) (1 - \sK_{\min}) + \sK_{\min} f(\frac{1}{\sK_{\min}}),
    \end{align*}
from which the desired result follows. 
\qedhere
\end{sketchofproof}

Next, with the minimax optimal mechanism given by \Cref{alg}, we examine the utility of our sampling approach compared to the existing state-of-the-art sampler.

\section{Experiments}

In this section, we compare the performance of our minimax optimal private sampler to the state-of-the-art private sampler, the relative mollifier framework (\Cref{def_RM}). The evaluation is done on synthetic data and two real-world datasets: Avezu CTR \citep{avazu_ctr} and MovieLens 1M \citep{movielens1m}. The code for experiments is available on Github \footnote{ \href{https://github.com/bzamanlooy/LDP-Sampling-Public-data}{https://github.com/bzamanlooy/LDP-sampling-public-data}}, and details on the computing resources are given in \Cref{ap_infrastructure}.

\subsection{Computation of Samplers}

For our minimax optimal sampler, given the public distribution $q$ and privacy parameter $\varepsilon$, we compute the optimal mechanism $\sK_{q, \varepsilon}$ using Algorithm~\ref{alg} and apply it to each individual's distribution $p_i$ to derive the corresponding sampling distribution $\hat{p}_i \coloneqq p_i \sK_{q,\eps}$.


As for the relative mollifier sampling framework (\Cref{def_RM}), we interpret the reference distribution in \Cref{def_RM} as the public prior. In their framework, we determine the normalizing constant $C$ described in \eqref{eq_information_projection_solution} using a bisection approach \citep{brent2013algorithms}. While \citet{husain2020local} focus exclusively on $\kl$-divergence in the relative mollifier framework, we extend their work to include the $\tv$-distance, offering an additional benchmark. This extension uses linear programming, with details provided in Appendix \ref{appendix_Husain_LP}.

The following sections outline the experimental setup and present the results for each dataset. Each subsection provides details on the public prior and individuals' distributions specific to the dataset.

\subsection{Synthetic Dataset}\label{Section_Numerical_verification}

We use synthetic data to simulate the local distribution of a single user and their public prior. Our goal is to examine how the gap between the user’s original distribution and the sampling distribution is affected by the user’s data.

In each round, the elements of the public distribution, $q$, are randomly sampled from the uniform distribution, $U[0,1]$, and are normalized to sum to 1. The local distribution $p = [p^1, \dots, p^n]$ is constructed by fixing the first element, $p^1$, and varying it from $1/n$ to 1, while keeping the remaining elements equal to ensure their total sum is also 1. As $p^1$ increases, $p$ becomes more concentrated, approximating a Dirac distribution. This allows us to observe scenarios where our approach, designed to minimize worst-case divergence, is expected to outperform the relative mollifier sampling framework due to \Cref{lemma_equivalence}.

In \Cref{fig_experiment_one_user}, we report the average $\tv$-distance and $\kl$-divergence between the user’s local and sampling distributions across 10 runs, along with standard errors, for a support size of $n = 100$ and privacy parameter $\varepsilon = 8$.


\begin{figure}[!htb]
   \begin{minipage}{0.5\textwidth}
     \centering
     \includegraphics[width=0.85\textwidth]{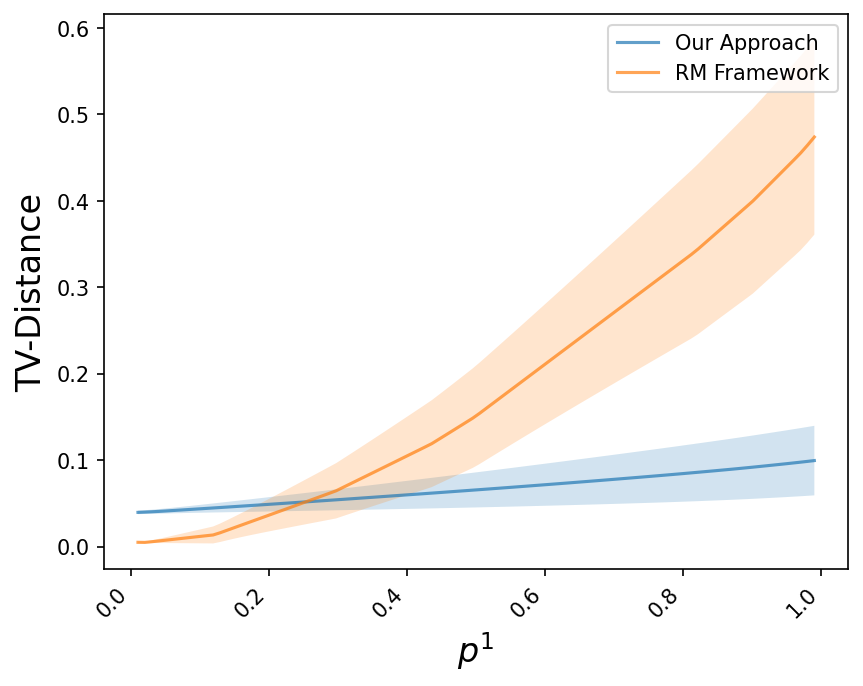}
   \end{minipage}\hfill
   \begin{minipage}{0.5\textwidth}
     \centering
     \includegraphics[width=0.85\textwidth]{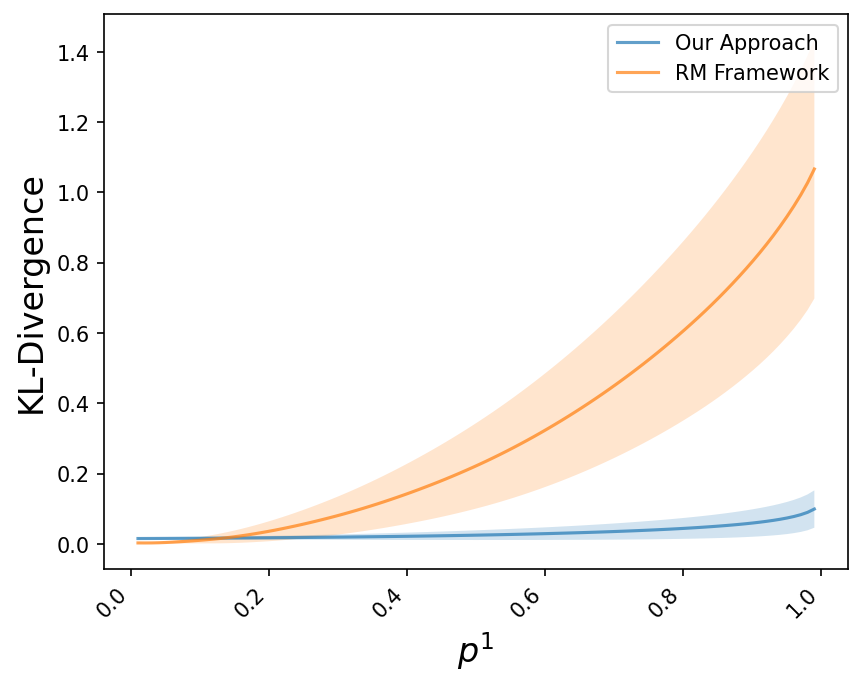}
   \end{minipage}
\captionsetup{skip=2pt} 
\caption{Comparison of our private sampling approach with the relative mollifier sampling framework using a simulated public prior. The local distribution $p$ is constructed by fixing $p^1$, with the rest identical and summing to 1. As $p^1$ increases, $p$ approaches a Dirac distribution, where our method is expected to outperform. Average $\tv$-distance and $\kl$-divergence are reported over 10 runs for $n = 100$ and $\varepsilon = 8$.}
    \label{fig_experiment_one_user}
\end{figure}



\Cref{fig_experiment_one_user} shows that when the user's distribution is closer to uniform (lower $p^1$), the relative mollifier sampling framework performs better. However, as $p^1$ increases and the user's local distribution approaches Dirac, our approach introduces significantly less distortion as expected.

 \subsection{Avezu Click Rate Prediction Dataset}


In this section, we demonstrate the practical improvements in a scenario where an AdTech company determines the next website to display an ad to a user, while protecting the privacy of the user’s local data.

\textbf{Dataset:} The dataset contains 40M rows with users' device IDs, the websites where ads were shown, and whether users clicked on them. It also includes nine anonymized user attributes (C1, C14–C21), each with multiple subcategories.

\textbf{Task:} The goal is to simulate an AdTech company selecting the optimal website to display an ad where a user is most likely to click, while ensuring the protection of their private click history. Each user's empirical distribution of their click history is treated as their individual distribution, \( p \). To approximate access to a public prior, we group users based on anonymized attributes and subcategories. The average empirical distribution of click histories within each subcategory serves as the public prior, \( q \), for each group.

We select the top 100 websites with the most clicks in each subcategory to predict the website most likely to receive a click from users. We keep users with at least 20 clicks on these websites and limit selection to subcategories with at least $100$ users. This ensures each subcategory represents a sufficiently large, identifiable demographic, where public data, such as aggregate statistics, would likely be available.

\textbf{Reporting Metric:}  
We calculate the $\tv$-distance between users' local distributions and their corresponding sampling distributions, obtained using our minimax optimal approach and the relative mollifier sampling framework, and report the maximum observed $\tv$-distance across all users for each attribute and subcategory.

\textbf{Results:} We present the results for the C18 category with $\varepsilon = 12$. As shown in \Cref{fig_avezu}, our approach significantly outperforms the relative mollifier sampling framework. For example, in subcategory 2, our method achieves a $9$-fold improvement in reducing the maximum $\tv$-distance.

Additionally, we provide a comparison across all attributes and subcategories for $\varepsilon = 8,12,$ and $16$, as shown in \Cref{additional_experiments}. When considering the results for all attributes and subcategories, our approach outperforms the relative mollifier sampling framework in 90.5\% of cases, with an average improvement of 0.46 in the maximum $\tv$-distance. In contrast, the relative mollifier sampling framework outperforms ours in 9.5\% of cases, with a smaller improvement of 0.009 in the maximum $\tv$-distance.

\begin{figure}[!htb]
    \centering
    \includegraphics[width=0.85\linewidth]{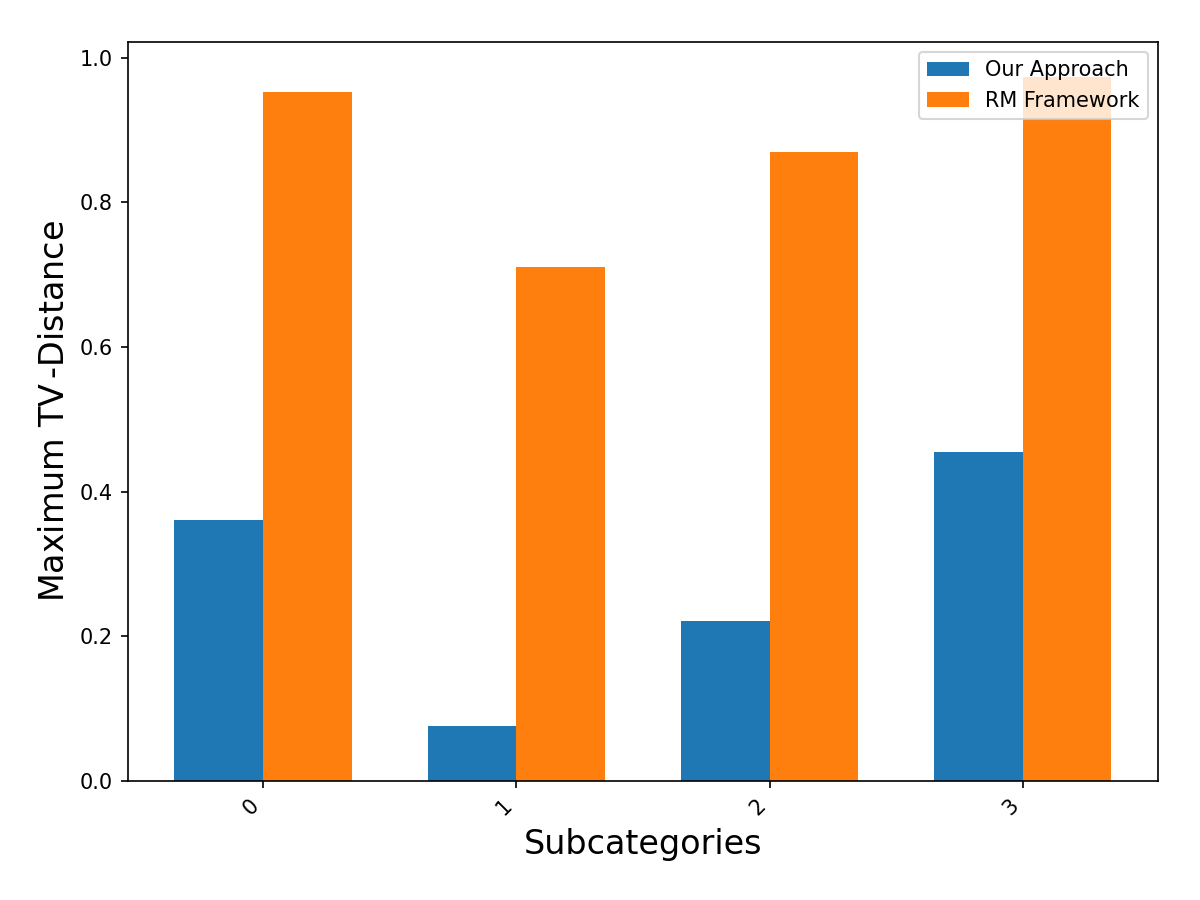}

\caption{Comparison of our private sampling method with the relative mollifier sampling framework  for inferring the next website for an AdTech company to display ads on. For each user, we first compute the $\tv$-distance between their local distribution and the corresponding sampling distribution. The figure then reports the maximum $\tv$-distance observed within each Subcategory for $\varepsilon = 12$.}
    \label{fig_avezu}

\end{figure}

\vspace{-15pt}

\subsection{MovieLens Datasets}




\textbf{Datasets:}  The MovieLens 100K and MovieLens 1M datasets contain 100,000 and 1 million ratings, respectively, provided by 1,000 and 6,000 users. Both datasets include demographic details, with our analysis focusing on age range.

\textbf{Task:}  
The goal is to simulate a streaming company's decision-making process for recommending the next movie genre to a user based on their age demographic while protecting their privacy.

For each user, we construct a local distribution, \( p \), by categorizing their rated movies based on the primary genre, summing the ratings within each genre, and normalizing these sums to obtain a probability distribution. With approximately 20 genres in the dataset, this results in a distribution over genre categories.  To estimate the public prior, \( q \), we group users from the MovieLens 100K dataset based on their age range and apply the same procedure within each group.

\textbf{Reporting Metric:}  
We calculate the $\tv$-distance between the local distributions of users and their corresponding sampling distributions and report the maximum among all users. This comparison is performed in age groups with a privacy parameter of $\eps = 5$.

\begin{figure}[!htb] 
\centering \includegraphics[width=0.85\linewidth]{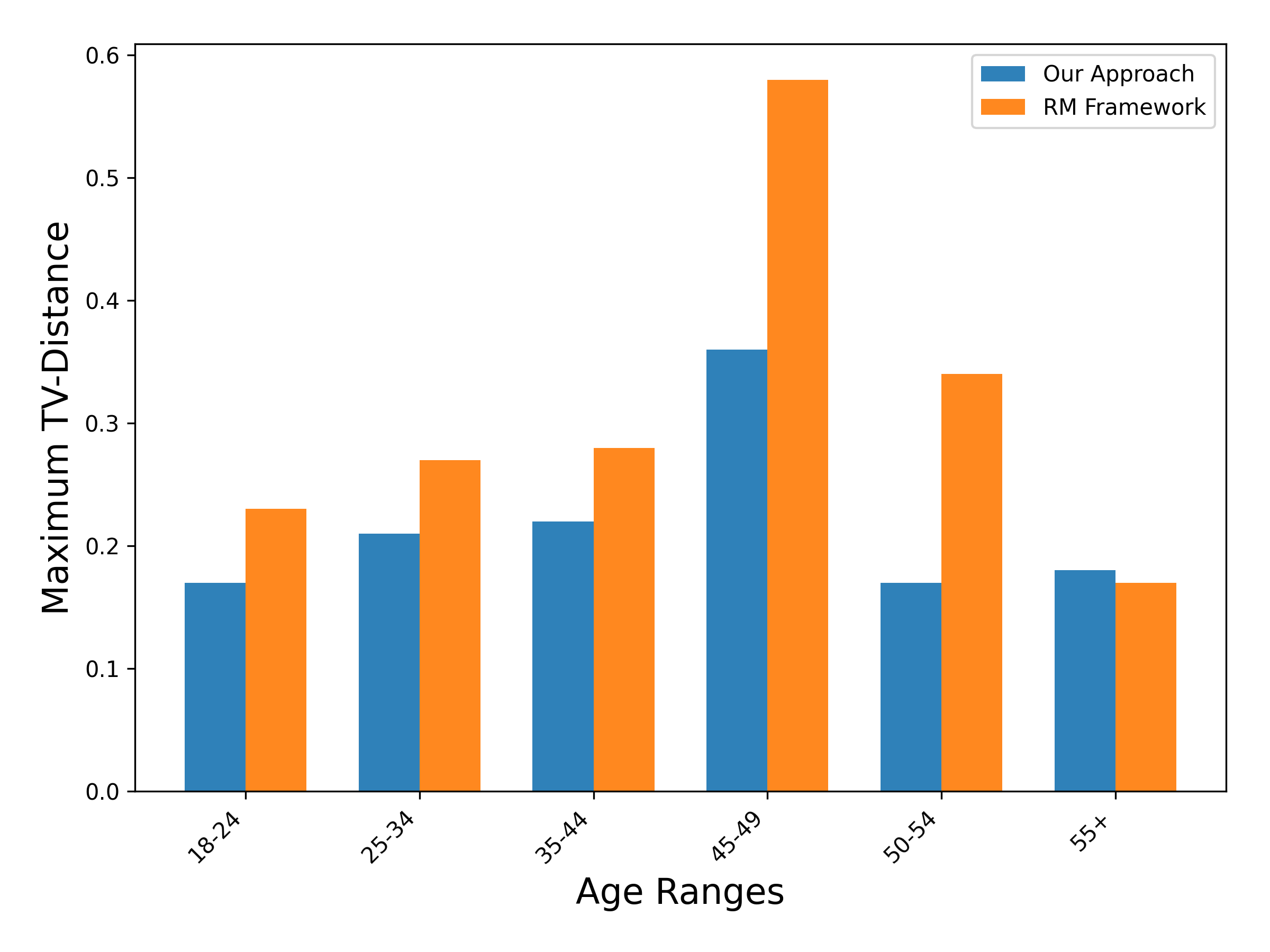} 

\caption{Comparison of our private sampling method with the relative mollifier sampling framework for inferring the genre of the next movie users are likely to watch. For each user, we first compute the $\tv$-distance between their local distribution and the corresponding sampling distribution. The figure then reports the maximum $\tv$-distance observed within each age group for $\varepsilon = 5$.}

\label{fig_movielens_main}

\end{figure}

\vspace{-10pt}
\textbf{Result:} As shown in \Cref{fig_movielens_main}, our method achieves a 41\% overall improvement in preserving the private distribution compared to the RM framework, demonstrating a substantial enhancement in maintaining the worst-off user's distribution.

Additionally, we present the results for $\varepsilon = 1$ to $8$, as shown in \Cref{additional_experiments_2}. When considering all privacy parameters, our approach outperforms the relative mollifier sampling framework in 54.1\% of cases, achieving an average improvement of 0.11 in the maximum $\tv$-distance. In contrast, the relative mollifier sampling framework outperforms our approach in 41.6\% of cases, with a much smaller average improvement of 0.03 in the maximum $\tv$-distance.


In summary, our experiments show that our minimax-optimal private sampler consistently outperforms the state-of-the-art approach across various scenarios, all while preserving the same level of privacy guarantees.

\section{Disclosure of Funding}

The work of B. Zamanlooy was supported in part by the NSERC Gradute Scholarsip.  The work of M. Diaz was supported in part by the Programa de Apoyo a Proyectos de Investigación e Innovación Tecnológica (PAPIIT) under grant IN103224. The work of S. Asoodeh was supported in part by the NSERC of Canada.

\newpage
\bibliography{references}
\bibliographystyle{plainnat}

\newpage
\section*{Checklist}



 \begin{enumerate}

 \item For all models and algorithms presented, check if you include:
 \begin{enumerate}
   \item A clear description of the mathematical setting, assumptions, algorithm, and/or model. [Yes/No/Not Applicable] Yes. We have described \Cref{alg} in 4 steps in section 2 and provided the pseudo code.
   \item An analysis of the properties and complexity (time, space, sample size) of any algorithm. [Yes/No/Not Applicable]
   Yes. The time complexity analysis of \Cref{alg} is given at the end of Section 3.
   \item (Optional) Anonymized source code, with specification of all dependencies, including external libraries. [Yes/No/Not Applicable]
   Yes - We have included the code with specification of all dependencies for all experiments in the supplementary material.
 \end{enumerate}

 \item For any theoretical claim, check if you include:
 \begin{enumerate}
   \item Statements of the full set of assumptions of all theoretical results. [Yes/No/Not Applicable]
   Yes. All assumptions are stated for each theoretical claim in the main body of the paper.
   \item Complete proofs of all theoretical results. [Yes/No/Not Applicable]
   Yes. These are provided in the appendix. Additionally, we provide a sketch of the proofs in the main body of the paper.
   \item Clear explanations of any assumptions. [Yes/No/Not Applicable]     Yes. All assumptions are stated for each theoretical claim in the main body of text.
 \end{enumerate}

 \item For all figures and tables that present empirical results, check if you include:
 \begin{enumerate}
   \item The code, data, and instructions needed to reproduce the main experimental results (either in the supplemental material or as a URL). [Yes/No/Not Applicable]
    Yes.
    \href{https://github.com/bzamanlooy/LDP-Sampling-Public-data}{https://github.com/bzamanlooy/LDP-Sampling-Public-data}
   \item All the training details (e.g., data splits, hyperparameters, how they were chosen). [Yes/No/Not Applicable] No Applicable.
         \item A clear definition of the specific measure or statistics and error bars (e.g., with respect to the random seed after running experiments multiple times). [Yes/No/Not Applicable] Yes. This is provided in the main body of the text in the experiments section.
         \item A description of the computing infrastructure used. (e.g., type of GPUs, internal cluster, or cloud provider). [Yes/No/Not Applicable] Yes. We provide it in \Cref{ap_infrastructure} and mention it in the main text.
 \end{enumerate}

 \item If you are using existing assets (e.g., code, data, models) or curating/releasing new assets, check if you include:
 \begin{enumerate}
   \item Citations of the creator If your work uses existing assets. Yes. We have cited the creators of the datasets we used in our experiments. 
   \item The license information of the assets, if applicable.  Not Applicable
   \item New assets either in the supplemental material or as a URL, if applicable. Not Applicable
   \item Information about consent from data providers/curators.  Not Applicable
   \item Discussion of sensible content if applicable, e.g., personally identifiable information or offensive content.  Not Applicable
 \end{enumerate}

 \item If you used crowdsourcing or conducted research with human subjects, check if you include:
 \begin{enumerate}
   \item The full text of instructions given to participants and screenshots. Not Applicable
   \item Descriptions of potential participant risks, with links to Institutional Review Board (IRB) approvals if applicable. Not Applicable
   \item The estimated hourly wage paid to participants and the total amount spent on participant compensation. Not Applicable
 \end{enumerate}

 \end{enumerate}

\clearpage
\onecolumn
\appendix

\section{Appendix}

\subsection{Preliminary Results}
We first present some results that are useful in showing our main results.
\begin{lemma}\label{lemma_g_decreasing}
 Given a convex function $f:[0, \infty) \rightarrow \mathbb{R}$ such that $f(1)=0$ and $f(0) < \infty$, the function 
 \begin{equation}
     g(x) = f(0) + x \left( f\left(1/x\right) -f(0)\right)
 \end{equation}

 is a non-increasing function on $0 \leq x \leq 1$. 
\end{lemma}

\begin{proof}
  It suffices to show that the function \( x \mapsto x \left( f\left(\frac{1}{x}\right) - f(0) \right) \) is decreasing for \( x \in [0,1] \).
  
  Notice that for any given \( x \in [0,1] \), the expression \( x \left( f\left( \frac{1}{x} \right) - f(0) \right) \) represents the slope of the line between the points \( (0, f(0)) \) and \( \left( \frac{1}{x}, f\left( \frac{1}{x} \right) \right) \) on the graph of \( f \). Since \( f \) is convex, the secant slope between these points is non-decreasing as \( \frac{1}{x} \) increases (which occurs when \( x \) decreases).

  Therefore, the function \( x \mapsto x \left( f\left(\frac{1}{x}\right) - f(0) \right) \) is decreasing on \( [0,1] \), implying that \( g(x) \) is non-increasing for \( x \in [0,1] \).
\end{proof}


\begin{lemma}[Upper Bound on the Diagonal]\label{lemma_UB_diagonal}
   Let $q \in \Delta(\X)$ and $q_{\min} = \min_{x \in \X} q(x)$. Given an $\eps$-LDP mechanism $\sK$ such that $q \sK = q$, we have that 
   \begin{equation}
       \sK_{ii} \leq \ri.
   \end{equation}
   and hence
    \begin{equation}
       \min_{i} \sK_{ii} \leq \r.
   \end{equation}
\end{lemma}

\begin{proof}
     Given an $\eps$-LDP mechanism $\sK$ such that $q \sK = q$, we have that $\sum_j q_j \sK_{ji} = q_i$  and hence for each element of the diagonal of $\sK$,  $\sK_{ii}$ we have that 
    \[  q_i - q_i \sK_{ii}  =  \sum_{j \neq i} q_j \sK_{ji}.\]
    Additionally, since $\sK$ is $\eps$-LDP, we have that $\sK_{ji} \geq e^{-\eps} \sK_{ii}$. Using this, we have that 
    \[
     q_i - q_i \sK_{ii} \geq  e^{-\eps} \sK_{ii} \sum_{j\neq i}q_j =  e^{-\eps} \sK_{ii} (1 - q_i).
    \] 
    By simple calculations, we have that $\sK_{ii} \left( q_i + e^{-\eps} ( 1  -  q_i) \right)\leq q_i$ which gives us
   \begin{equation}
       \sK_{ii} \leq \ri.
   \end{equation}
   Additionally, since $x \mapsto \frac{e^\eps x}{e^\eps x + 1 - x}$ is an increasing function on $0 < x \leq 1$, we have that 
    \begin{equation}\label{eq_f_div_equivalence_fourth_app}
        \min_{i} \sK_{ii} \leq \frac{e^\eps q_{\min}}{e^\eps q_{\min} + 1 - q_{\min}},
    \end{equation}
    as we wanted to prove.
\end{proof}

\subsection{Proof of \Cref{lemma_equivalence}}
\begin{proof}
 Since $\df$ is jointly convex, the map $p \xrightarrow{} \df(p \| p \sK )$ is convex on $\Delta(\X)$. Hence, $ \sup_{p \in \Delta(\X)} \df (p \| p \sK)$ is attained at one of the extremes of $\Delta(\X)$ which are the Dirac distributions $\delta_i$ with their $i$th element equal to $1$.   Therefore, we have that
 \begin{align}
     \sup_{p \in \Delta(\X)} \df(p \| p \sK )
     &= \max_{i} \df(\delta_i \| \delta_i \sK) \label{eq_lemma1_1}.\\ \label{eq_lemma1_2}
 \end{align}
    By  expanding the definition of $f$-divergence for discrete distributions we have that
    \begin{equation}\label{eq_f_div_equivalence_second}
      D_f(\delta_i \| \delta_i \sK) 
      = \sum_{x \in \X} \delta_i \sK(x) f(\frac{\delta_i (x)}{\delta_i \sK(x)}) 
      = \sK_{ii} f(\frac{1}{\sK_{ii}}) + f(0) \sum_{j \neq i} \sK_{ij}.
    \end{equation}
    Using \Cref{eq_f_div_equivalence_second} and since $\sK$ is row-stochastic we have that
    \begin{equation}\label{eq_f_div_equivalence_third}
    D_f(\delta_i \| \delta_i \sK)  = g(\sK_{ii}) = f(0) + \sK_{ii} (f(\frac{1}{\sK_{ii}}) - f(0)).
    \end{equation}
    where $g(x) = f(0) + x (f(1/x) -f(0))$.
Then, by invoking \eqref{eq_f_div_equivalence_third} and \Cref{lemma_g_decreasing}, we have that 
    \begin{align}
        \sup_{p \in \Delta(\X)} D_f(p \| p \sK) 
        &= \max_{i} g(\sK_{ii}) \\
        &= g(\min_{i} \sK_{ii}) \\ \label{eq_f_div_equal_g_app}
    \end{align}
and we have the result.
\end{proof}

\subsection{Proof of \Cref{proposition_converse}}

\begin{proof}
    
    By \Cref{lemma_equivalence} we have that
    \begin{equation}
        \sup_{p \in \Delta(\X)} \df (p \| p \sK) = g(\min_{i} \sK_{ii})
    \end{equation}
    where $g(x) = f(0) + x (f(1/x) -f(0))$.
    Additionally, due to \Cref{lemma_UB_diagonal},
    \begin{equation}\label{eq_f_div_equivalence_fourth}
        \min_{i} \sK_{ii} \leq \frac{e^\eps q_{\min}}{e^\eps q_{\min} + 1 - q_{\min}}.
    \end{equation}
    Furthermore, since $g$ is a decreaing function on $[0,1]$ (\Cref{lemma_g_decreasing}), we have that 
    \begin{equation}\label{eq_f_div_equivalence_fifth}
        g(\min_{i} \sK_{ii}) \geq g(\frac{e^\eps q_{\min}}{e^\eps q_{\min} + 1 - q_{\min}}).
    \end{equation}
     Then, by invoking \eqref{eq_f_div_equivalence_third}, \Cref{lemma_g_decreasing}, and \eqref{eq_f_div_equivalence_fifth}, we have that 
    \begin{align}
        \sup_{p \in \Delta(\X)} D_f(p \| p \sK) 
        &= g(\min_{i} \sK_{ii}) \\ \label{eq_f_div_equal_g}
        & \geq g(\frac{e^\eps q_{\min}}{e^\eps q_{\min} + 1 - q_{\min}}) \\
        &= \frac{1-q_{\min}}{e^\eps q_{\min} + 1 - q_{\min}} f(0) + \frac{e^\eps q_{\min}}{e^\eps q_{\min} + 1 - q_{\min}} f(\frac{e^\eps q_{\min} + 1 - q_{\min}}{e^\eps q_{\min}}).
    \end{align}
This gives us the lower bound.
\end{proof}

\subsection{Proof of \Cref{lemma_binary}}

\begin{proof}
Define
\begin{equation}
    \sKo =\frac{1}{e^\eps \alpha + 1 - \alpha } \begin{bmatrix}
    e^\eps \alpha      & 1 - \alpha \\
    \alpha      & (e^\eps -1) \alpha + 1 - \alpha
\end{bmatrix} .
\end{equation}

Then by particularizing the lower bound in \Cref{proposition_converse}, we have that 
\begin{equation}
   \optutil \geq 
   \frac{1-\alpha}{e^\eps \alpha + 1 - \alpha} f(0) + \frac{e^\eps \alpha}{e^\eps \alpha + 1 - \alpha} f(\frac{e^\eps \alpha + 1 - \alpha}{e^\eps \alpha}).
\end{equation}
 Define $g(x) = f(0) + x (f(1/x) -f(0))$. Then, due to \eqref{eq_f_div_equal_g}, we have that
\begin{align}
    \sup_{p \in \Delta(\X)} D_f(p \| p \sKo) 
    &=  g(\min_{i} \ \sKp_{ii}) \\
    &= g(\sKp_{11}) \\
    &=    \frac{1-\alpha}{e^\eps \alpha + 1 - \alpha} f(0) + \frac{e^\eps \alpha}{e^\eps \alpha + 1 - \alpha} f(\frac{e^\eps \alpha + 1 - \alpha}{e^\eps \alpha}).
\end{align}
Hence,
\begin{equation}
    \optutil = \sup_{p \in \Delta(\X)} D_f(p \| p \sKo),
\end{equation}
and we have the result.
\end{proof}

\subsection{Proof of \Cref{lemma_K_consistency} }

\begin{proof}
We will prove by induction that the optimal mechanism $\sKo$ is row stochastic, $\eps$-LDP, and satisfies the constraint $q \sKo = q$.
Let $\sK \coloneqq \sK_{q,\eps}$ be the mechanism obtained from \Cref{alg}. In particular, we can write $\sK_{q,\eps}$ as
\begin{equation} \label{eq_matrix_2}
\sK = 
\begin{bmatrix}
  \frac{e^\eps q_1}{e^\eps q_1 + 1 - q_1} & \frac{q_2}{e^\eps q_1 + 1 - q_1} & \dots& \frac{q_n}{e^\eps q_1 + 1 - q_1} \\
  \frac{ q_1}{e^\eps q_1 + 1 - q_1} & &  \\
  \dots & & m\sM& \\
  \frac{ q_1}{e^\eps q_1 + 1 - q_1} &  & 
\end{bmatrix}.
\end{equation}

where $\sM \coloneqq \sK_{\bar{q}, \eps}$ is the optimal mechanism for \Cref{eq_original} with the center point $\bar{q} = \frac{1}{\sum_{i=2}^n q_i} [q_2, \dots, q_n]$ and privacy budget $\eps$. Additionally,  $m = 1 - \frac{ q_1}{e^\eps q_1 + 1 - q_1}$ re-scales the elements of $\sM$ to make sure $\sK$ is row stochastic. Indeed, 
%
%
one can easily confirm that $\sK$ in \eqref{eq_matrix_2} is row-stochastic. This is clear for the first row. For the rest of the rows $i \neq 1$, we have that
\[
\sum_{j=1}^n \sK_{ij} = \frac{q_1}{e^\eps q_1 + 1 - q_1} + m \sum_{j = 1}^{n-1} M_{(i-1)j} = \frac{q_1}{e^\eps q_1 + 1 - q_1} + (1 - \frac{q_1}{e^\eps q_1 + 1 - q_1}) 1 = 1.
\]

Additionally, we can show that $q \sK = q$. For the first column,
\[
\sum_{i=1}^n q_i \sK_{i1} = q_1 \frac{e^\eps q_1 + \sum_{i=2}^n q_i}{ e^\eps q_1 + 1 - q_1} = q_1 \frac{e^\eps q_1 + 1 - q_1}{e^\eps q_1 + 1 - q_1} = q_1.
\]

For the other columns $j \neq 1$, we have that

\begin{equation*}
    \begin{aligned}
        \sum_{i=1}^n q_i \sK_{ij}
        &= \frac{q_1 q_j}{e^\eps q_1 + 1 - q_1} + \sum_{i=2}^{n} q_i m \sM_{(i-1)(j-1)} \\
        &= \frac{q_1 q_j}{e^\eps q_1 + 1 - q_1} + (1 - q_1) m \sum_{i=2}^{n} \frac{q_i}{1-q_1} \sM_{(i-1)(j-1)}  \\
        &= \frac{q_1 q_j}{e^\eps q_1 + 1 - q_1} + (1-q_1) m  \frac{q_j}{1-q_1} \\
        &= \frac{q_1 q_j}{e^\eps q_1 + 1 - q_1} +  ( 1 - \frac{q_1}{e^\eps q_1 + 1 - q_1}) q_j\\
        &= q_j (\frac{q_1}{e^\eps q_1 + 1 - q_1} + 1 - \frac{q_1}{e^\eps q_1 + 1 - q_1} ) \\
        &= q_j.
    \end{aligned}
\end{equation*}

So we need to prove that $\sK$ is $\eps$-LDP. This is done by induction where the basis is the validity of the optimal mechanism for binary priors in \Cref{lemma_binary}.

\textbf{Base case: }
The $\eps$-LDP constraints hold for the first column of $\sK$. For the rest of the columns, we only need to check the $\eps$-LDP constraints for $\sK_{1j}$ for $j \in \{2, \dots, n\}$, since $\sM$ is $\eps$-LDP by induction. In particular, we need to show that 
\begin{equation*}
m e^{-\eps} \max \{ \sM_{1(j-1)},\dots ,\sM_{(n-1)(j-1)} \}  \leq \sK_{1j} \leq m e^{\eps} \min \{ \sM_{1(j-1)},\dots ,\sM_{(n-1)(j-1)} \}    
\end{equation*}
for $j \in \{2, \dots, n\}.$

Without loss of generality, Assume $j =2$, then we should prove

\begin{equation} \label{eq_LDP_constraint_2}
m e^{-\eps} \max \{ \sM_{11},\dots ,\sM_{(n-1)1} \}  \leq \sK_{12} \leq m e^{\eps} \min \{ \sM_{11},\dots ,\sM_{(n-1)1}. \}    
\end{equation}




Notice that for each $i \in \{2, \dots, n-1\}$, we have that 
\begin{equation} \label{eq_LDP_UB_single_2}
\sM_{i1} \leq e^\eps \sM_{\sM_{11}}.    
\end{equation}
Additionally, we have that $\sum_{i = 1}^{n-1} q_{i+1} \sM_{i1} = q_2$. Putting these two together, we have that
\[
\sM_{i1} \geq  \frac{q_2}{q_{i+1} + \sum_{j=2, j \neq i+1}^{n} e^\eps q_j}.
\]
Since q is sorted in an increasing order, we have that
\[
\min\{\sM_{11}, \dots, \sM_{(n-1)1} \} \geq \frac{q_2}{q_2 +  e^\eps \sum_{i=3}^{n}  q_i}.
\]

Hence to show the upper bound constraint in \eqref{eq_LDP_constraint_2}, it suffices to show that 
\begin{equation}\label{eq_LDP_UB_2}
\sK_{12} \leq m e^\eps \frac{q_2}{q_2 +  e^\eps \sum_{i=3}^{n}  q_i}.    
\end{equation}






With Some algebraic manipulation, we can show that \eqref{eq_LDP_UB_2} is equivalent to 
\[
0 \leq e^\eps q_1 (e^\eps - 1) + q_2 (e^\eps - 1),
\]
which is true since $e^\eps \geq 1$.



So we have proven \eqref{eq_LDP_UB_2}. With minimal changes, we can establish
\begin{equation*}
\sK_{1j} \leq m e^{\eps} \min \{ \sM_{1(j-1)},\dots ,\sM_{(n-1)(j-1)} \},    
\end{equation*}
for $j \in \{2, \dots, n\}.$

To prove the lower bound constraint in \eqref{eq_LDP_constraint_2}, we replace \eqref{eq_LDP_UB_single_2} with $\sM_{i1} \geq e^{-\eps} \sM_{11}$ and find that
\[
\max \{ \sM_{11},\dots ,\sM_{(n-1)1} \} \leq \frac{e^\eps q_2}{e^\eps q_2 + \sum_{i=3}^{n} q_i}.
\]

Hence to prove the lower bound constraint in \eqref{eq_LDP_UB_single_2}, we can equivalently prove
\begin{equation}\label{eq_LDP_LB_2}
m e^{-\eps}\frac{e^\eps q_2}{e^\eps q_2 + \sum_{i=3}^{n} q_i} \leq \sK_{12}.
\end{equation}
With some algebraic manipulation, we can show that proving \eqref{eq_LDP_LB_2} is equivalent to proving $q_1 (e^\eps -1) \leq q_2 (e^\eps - 1)$ which is true since q is sorted.

With minimal changes, we can establish

\begin{equation*}
m e^{-\eps} \max \{ \sM_{1(j-1)},\dots ,\sM_{(n-1)(j-1)} \}  \leq \sK_{1j},    
\end{equation*}
for $j \in \{2, \dots, n\}.$
So, we have proven 
\begin{equation*}
m e^{-\eps} \max \{ \sM_{1(j-1)},\dots ,\sM_{(n-1)(j-1)} \}  \leq \sK_{1j} \leq m e^{\eps} \min \{ \sM_{1(j-1)},\dots ,\sM_{(n-1)(j-1)} \},   
\end{equation*}
for $j \in \{2, \dots, n\}.$

Hence the mechanism $\sK$ in \eqref{eq_matrix_2} is $\eps$-LDP with $q \sK = \sK$. 
\end{proof}

\subsection{Proof of \Cref{thrm_optimal_mechanism_part_2}}

\begin{proof}
    Assume that $f(0) < \infty$, since otherwise $\optutil = \infty$ and we are not interested in characterizing the optimal mechanism.
    
    Now, for ease of notation let $\sK \coloneqq \sK_{q,\eps}$. By induction, we first show that $\sK_{11} \leq  \sK_{22} \leq \dots \leq \sK_{nn}$.
    \textbf{Base case.} Notice for $n=2$, $\sK_{11} 
    \leq \sK_{22}$.

    \textbf{Induction Step.} Let $\bar{q} = \frac{1}{\sum_{i=2}^n q_i}[q_2, \dots, q_n]$ and $\sM \coloneqq \sK_{\bar{q}, \eps}$. Now, assume that $\sM_{11} \leq \sM_{22} \leq \dots \leq \sM_{(n-1)(n-1)}$. Define $m = 1 - \frac{q_{\min}}{e^\eps q_{\min} + 1 - q_{\min}}$. Since $\sK_{ii} = m \sM_{(i-1)(i-1)}$ for $i > 1 $, by construction of $\sK$ from \Cref{alg}, we have that $\sK_{22} \leq \sK_{33} \leq \dots \leq \sK_{nn}$. So, all that is left to show is that $\sK_{11} \leq \sK_{22}$.
    So equivalently we want to show that
    \begin{align} \label{eq_K_order}
        \frac{e^{\eps} q_{\min}}{ e^\eps {q_{\min}} + 1 - q_{\min}} - 
        m \frac{e^\eps \frac{ q_2}{1 - q_{\min}}}{e^\eps\frac{ q_2}{1 - q_{\min}} + 1 - \frac{q_2}{1- q_{\min}}} \leq 0.
    \end{align}
    But showing \eqref{eq_K_order} is equivalent to showing:
    \begin{align}
         e^\eps q_{\min} - \frac{e^\eps q_{\min} + 1 - 2 q_{\min}}{e^\eps q_2 + 1 - q_{\min} - q_2} e^\eps q_2 \leq 0
    \end{align}
    which is true since $e^\eps q_{\min} (1- q_{\min}) - e^\eps q_2 (1 - q_{\min}) \leq 0.$
    Hence, we have that 
    \begin{equation}
       \sK_{11} \leq \sK_{22} \leq \dots \leq  \sK_{nn}.
    \end{equation}
Hence, by \Cref{lemma_equivalence} and since $\sKp_{11} = \r$, we have that
\begin{align}
  \df(p, p \sKo) 
  &= f(0) + \sKp_{11} (f(\frac{1}{\sKp_{11}}) -f(0)) \\
  &= \frac{1-q_{\min}}{e^\eps q_{\min} + 1 - q_{\min}} f(0) + \frac{e^\eps q_{\min}}{e^\eps q_{\min} + 1 - q_{\min}} f(\frac{e^\eps q_{\min} + 1 - q_{\min}}{e^\eps q_{\min}}),
\end{align}
as required.
\end{proof}

\subsection{Proof of \Cref{theorem_uniform}}

\begin{proof}

    By induction, we prove that if $q_{\text{u}}$ is the uniform prior, then an optimal mechanism, $\sKou$ , achieving $\Gamma(q_{\text{u}}, \eps)$ is the n-ary randomized response mechanism.

    \textbf{Base Case.} Notice that the mechanism characterizing the optimal mechanism for binary $q_{\text{u}}$ in \Cref{lemma_binary} is the randomized response mechanism.
    
    \textbf{Induction Step.} Notice that since $q_{\text{u}}$ is the uniform distribution then ${q_{\text{u}}}_{\min} = \frac{1}{n}$. Hence we have that 
    \begin{align}
        \sKpu_{11} &= \r = \frac{e^\eps}{e^\eps + n -1}\\
        \intertext{and,}
        \sKpu_{j1} &= \sKpu_{1j} = \frac{1}{e^\eps + n -1}
    \end{align}
    for $j > 1$.

    \begin{figure}[H]
    \centering    
    \[ \sKou = 
    \begin{bmatrix}
      \frac{e^\eps}{e^\eps + n -1} & \begin{matrix} \frac{1}{e^\eps + n -1} & \dots & \frac{1}{e^\eps + n -1} \end{matrix} \\
      \begin{matrix}  \frac{1}{e^\eps + n -1}\\ \vdots \\  \frac{1}{e^\eps + n -1} \end{matrix}  &
      \begin{bmatrix}
      \hspace*{-\arraycolsep}
      \phantom{e^\eps + n} & \phantom{e^\eps + n} & \phantom{e^\eps + n}
      \hspace*{-\arraycolsep}
      \\
      & \raisebox{-0.2\height}[0pt][0pt]{\large$m \sK_{\bar{q}_{\text{u}}, \eps}$} & \\
      & &
      \end{bmatrix}
    \end{bmatrix}
    \]
        \caption{One iteration of the recursion of Algorithm 1.}
        \label{fig_recursion_app}
\end{figure}
So, one-step recursion of \Cref{alg} for $\sKou$ looks like \Cref{fig_recursion_app}. Now assume that $\sK_{\bar{q}_{\text{u}},\eps}$ in \Cref{fig_recursion_app} is the $(n-1)$-ary randomized response, i.e all the elements of its diagonal are $\frac{e^\eps}{e^\eps + n -2}$ and all its non-diagonal elements are $\frac{1}{e^\eps + n -2}$. Multiplying these elements by $m = 1 - \frac{1}{e^\eps + n -1} = \frac{e^\eps + n - 2}{e^\eps + n - 2}$, we have that for $i \in [n-1]$,
\begin{align}
    \sKpu_{ii} &= \frac{e^\eps}{e^\eps + n -1},\\
    \intertext{and}
    \sKpu_{ij} &= \frac{1}{e^\eps + n -1}, & 
\end{align}
$\text{for } i \neq 1 \text{ and } i \in [n-1].$
Hence, $\sKou$ is the $n$-ary randomized reposne. 
Now using \Cref{thrm_optimal_utility}, we have that 
\begin{align}
    \Gamma(q_{\text{u}}, \eps) 
    &= \sup_{p \in \Delta(\X)} \df(p \| p \sKou) \\
    &= \sup_{p \in \Delta(\X)}  \df(p \| p\sK_{\text{RR}}) \\
    &= \frac{n-1}{e^\eps+k-1} f(0)+\frac{e^\eps}{e^\eps+n-1} f\left(\frac{e^\eps+n-1}{e^\eps}\right),
\end{align}
as we wanted to show. 
\end{proof}


\subsection{Linear Optimization for \cite{husain2020local}}\label{appendix_Husain_LP}
Given a public prior $q$ and a distribution with sensitive information $p$, we seek to find:

\begin{equation}
\inf_{\hat{p} \in \mathcal{M}_{\eps,q}} \tv(p, \hat{p}).   
\end{equation}

 This problem can be reformulated as the following linear program:

\begin{align*}
\text{minimize} \quad & \sum_{i=1}^{n} Z_i \\
\text{subject to} \quad & \sum_{i=1}^{n} \hat{p}_i = 1 \\
& \hat{p}_i \leq q_i \cdot e^{\frac{\varepsilon}{2}}, \quad \forall i \\
& \hat{p}_i \cdot e^{\frac{\varepsilon}{2}} \geq q_i, \quad \forall i \\
& Z_i \geq 0, \quad \forall i \\
& Z_i \geq p_i - \hat{p}_i, \quad \forall i \\
& \hat{p}_i \geq 0, \quad \forall i \\
& \hat{p}_i \leq 1, \quad \forall i
\end{align*}

Where $\hat{p}$ is the projected distribution. This linear program is solved using the Python interface of the Gurobi optimization package using their free license.



\subsection{Additional Experiments - Avezu Click Rate Dataset}\label{additional_experiments}

\begin{figure}[H]
    \centering
    \includegraphics[width=0.8\textwidth]{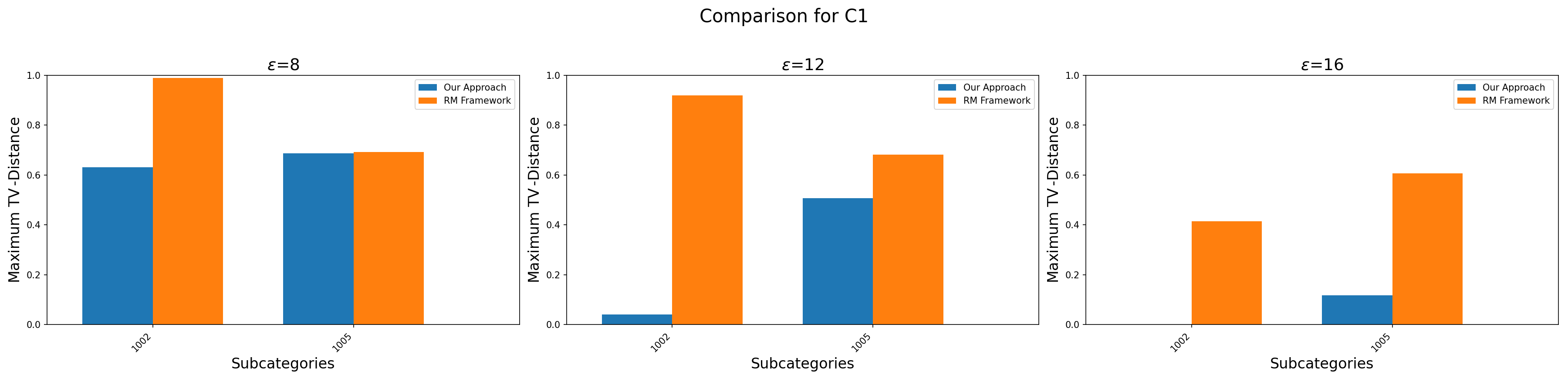}
    \caption{Comparison of our private sampling method with the relative mollifier sampling framework for predicting the next website for the AdTech company to display ads on. The figure reports the maximum $\tv$-distance between users' local distributions and their corresponding sampling distributions. Results are shown for attribute \texttt{C1} and all its subcategories with privacy parameters of $\varepsilon = 8, 12, 16$.}
    \label{fig:C1_plots}
\end{figure}

\begin{figure}[H]
    \centering
    \includegraphics[width=0.8\textwidth]{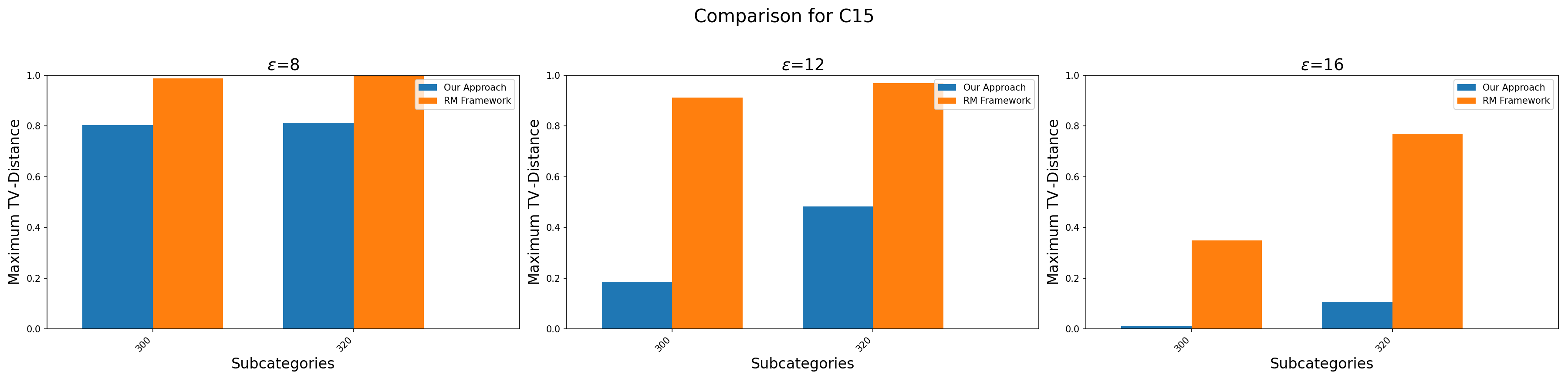}
    \caption{Comparison of our private sampling method with the relative mollifier sampling framework for predicting the next website for the AdTech company to display ads on. The figure reports the maximum $\tv$-distance between users' local distributions and their corresponding sampling distributions. Results are shown for attribute \texttt{C15} and all its subcategories with privacy parameters of $\varepsilon = 8, 12, 16$.}
    \label{fig:C15_plots}
\end{figure}

\begin{figure}[H]
    \centering
    \includegraphics[width=0.8\textwidth]{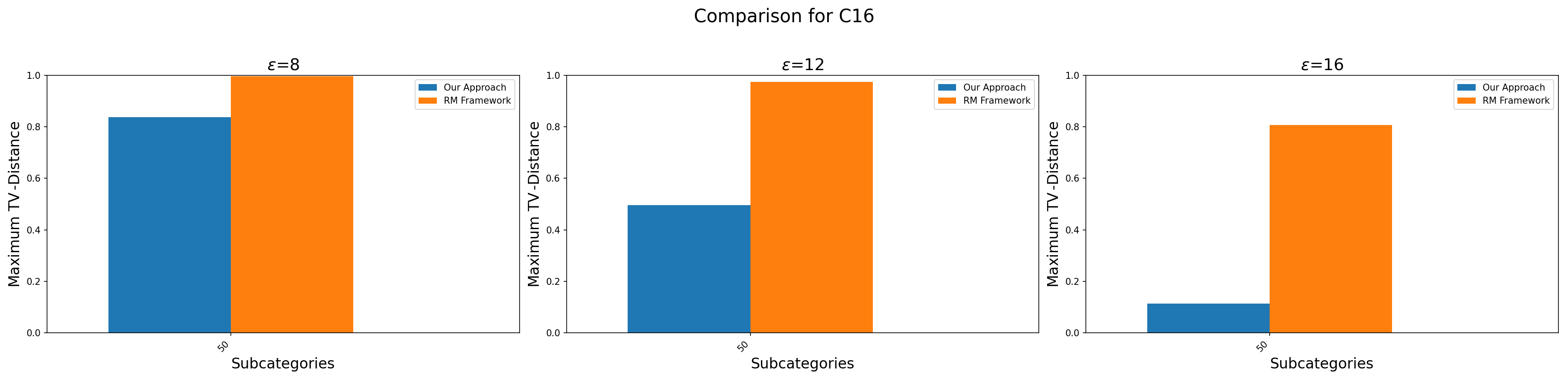}
    \caption{Comparison of our private sampling method with the relative mollifier sampling framework for predicting the next website for the AdTech company to display ads on. The figure reports the maximum $\tv$-distance between users' local distributions and their corresponding sampling distributions. Results are shown for attribute \texttt{C16} and all its subcategories with privacy parameters of $\varepsilon = 8, 12, 16$.}
    \label{fig:C16_plots}
\end{figure}

\begin{figure}[H]
    \centering
    \includegraphics[width=0.8\textwidth]{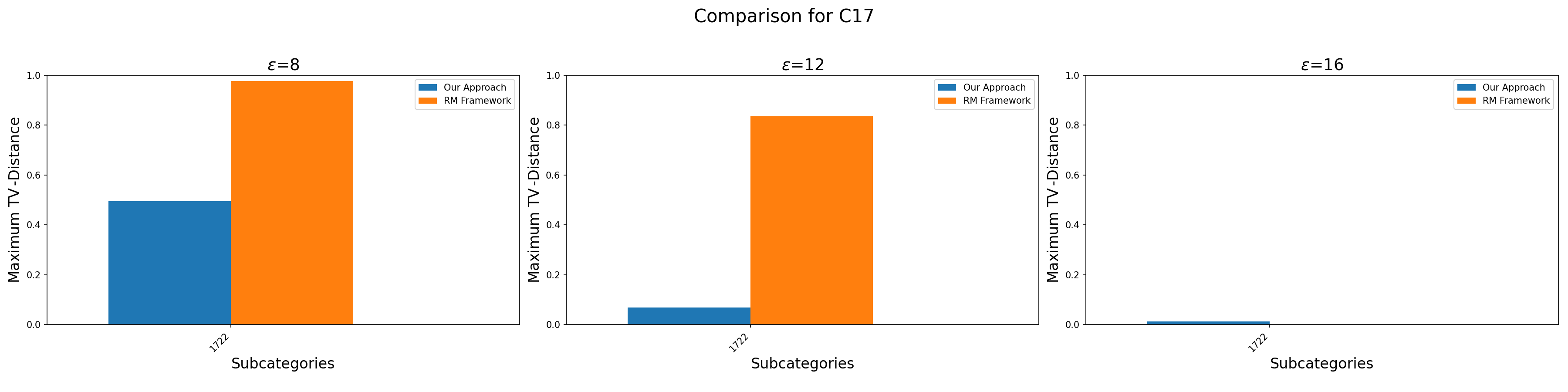}
    \caption{Comparison of our private sampling method with the relative mollifier sampling framework for predicting the next website for the AdTech company to display ads on. The figure reports the maximum $\tv$-distance between users' local distributions and their corresponding sampling distributions. Results are shown for attribute \texttt{C17} and all its subcategories with privacy parameters of $\varepsilon = 8, 12, 16$.}
    \label{fig:C17_plots}
\end{figure}

\begin{figure}[H]
    \centering
    \includegraphics[width=0.8\textwidth]{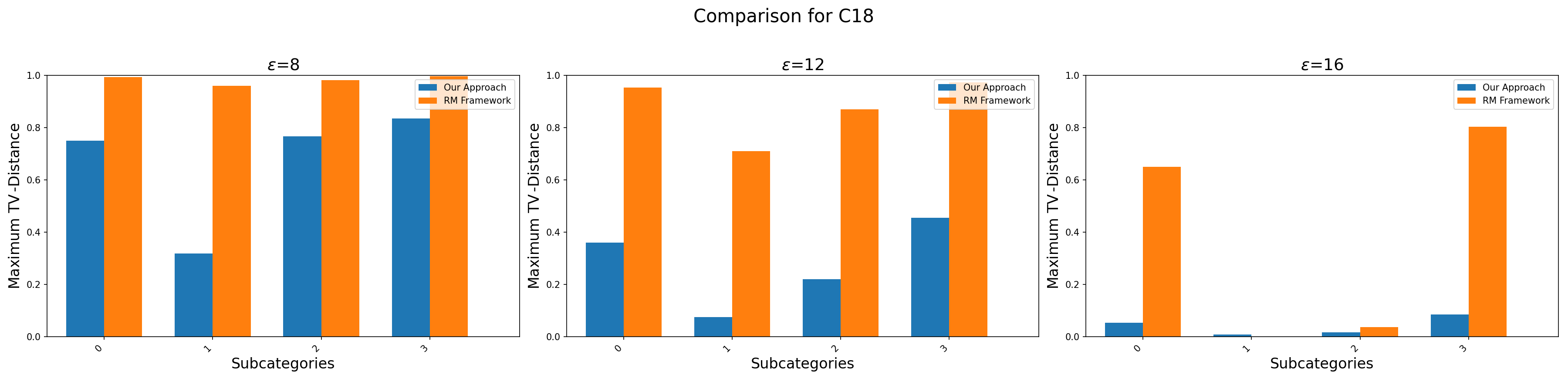}
    \caption{Comparison of our private sampling method with the relative mollifier sampling framework for predicting the next website for the AdTech company to display ads on. The figure reports the maximum $\tv$-distance between users' local distributions and their corresponding sampling distributions. Results are shown for attribute \texttt{C18} and all its subcategories with privacy parameters of $\varepsilon = 8, 12, 16$.}
    \label{fig:C18_plots}
\end{figure}

\begin{figure}[htbp]
    \centering
    \includegraphics[width=0.8\textwidth]{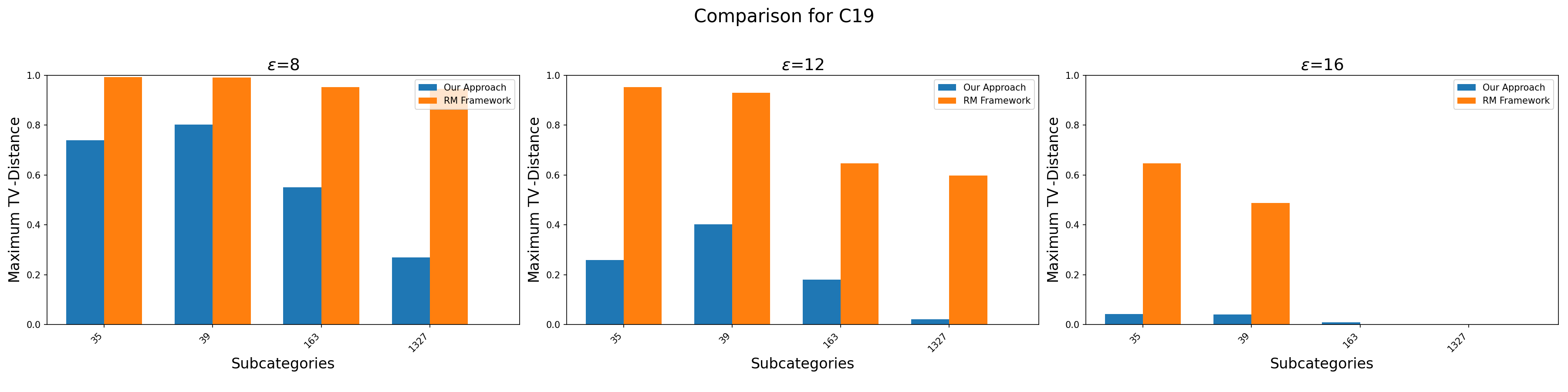}
    \caption{Comparison of our private sampling method with the relative mollifier sampling framework for predicting the next website for the AdTech company to display ads on. The figure reports the maximum $\tv$-distance between users' local distributions and their corresponding sampling distributions. Results are shown for attribute \texttt{C19} and all its subcategories with privacy parameters of $\varepsilon = 8, 12, 16$.}
    \label{fig:C19_plots}
\end{figure}

\subsection{Additional Experiments - MovieLens Dataset}\label{additional_experiments_2}

\begin{figure}[H]
    \centering
    \subfloat{\includegraphics[width=0.4\linewidth]{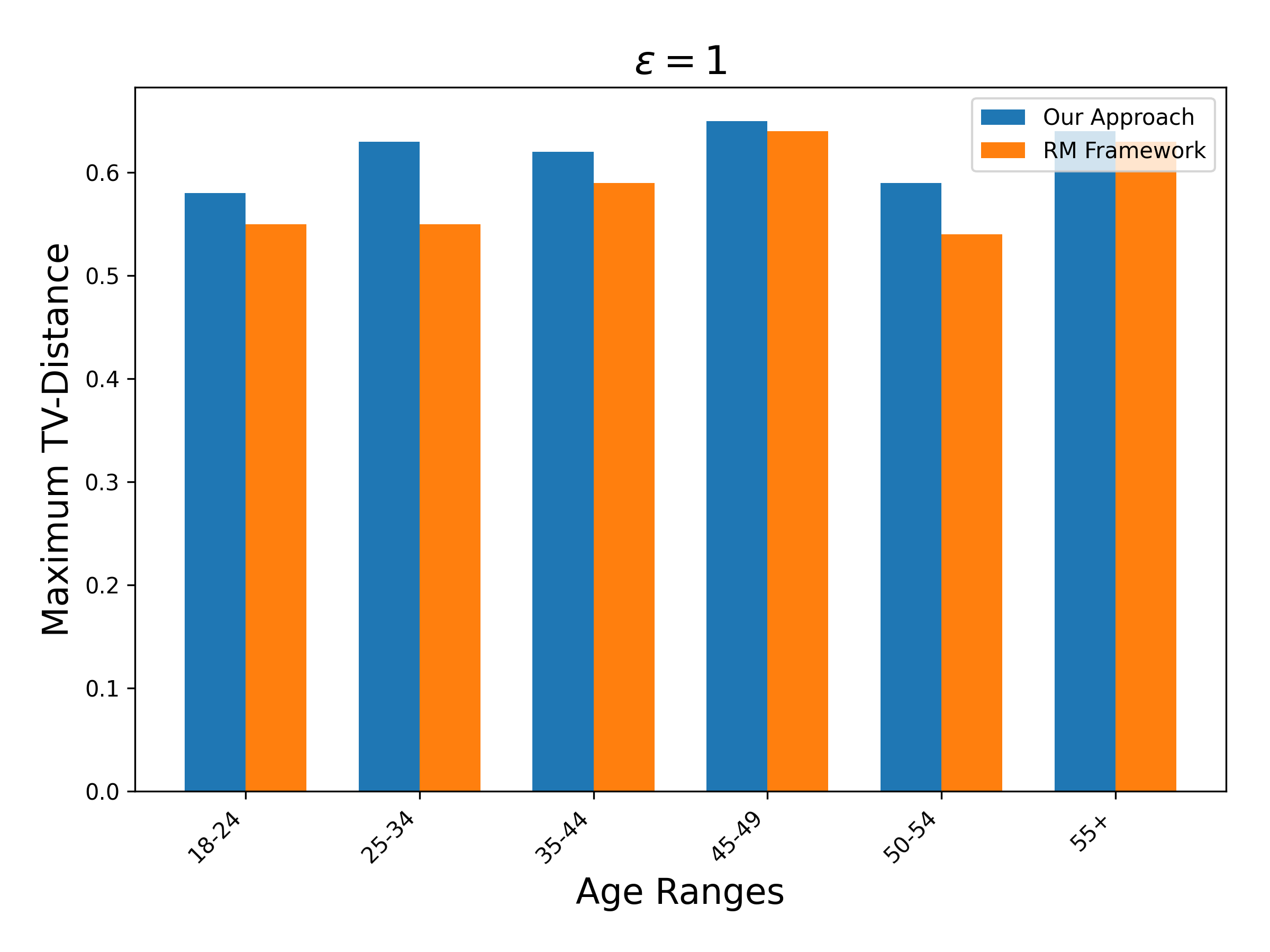}}
    \subfloat{\includegraphics[width=0.4\linewidth]{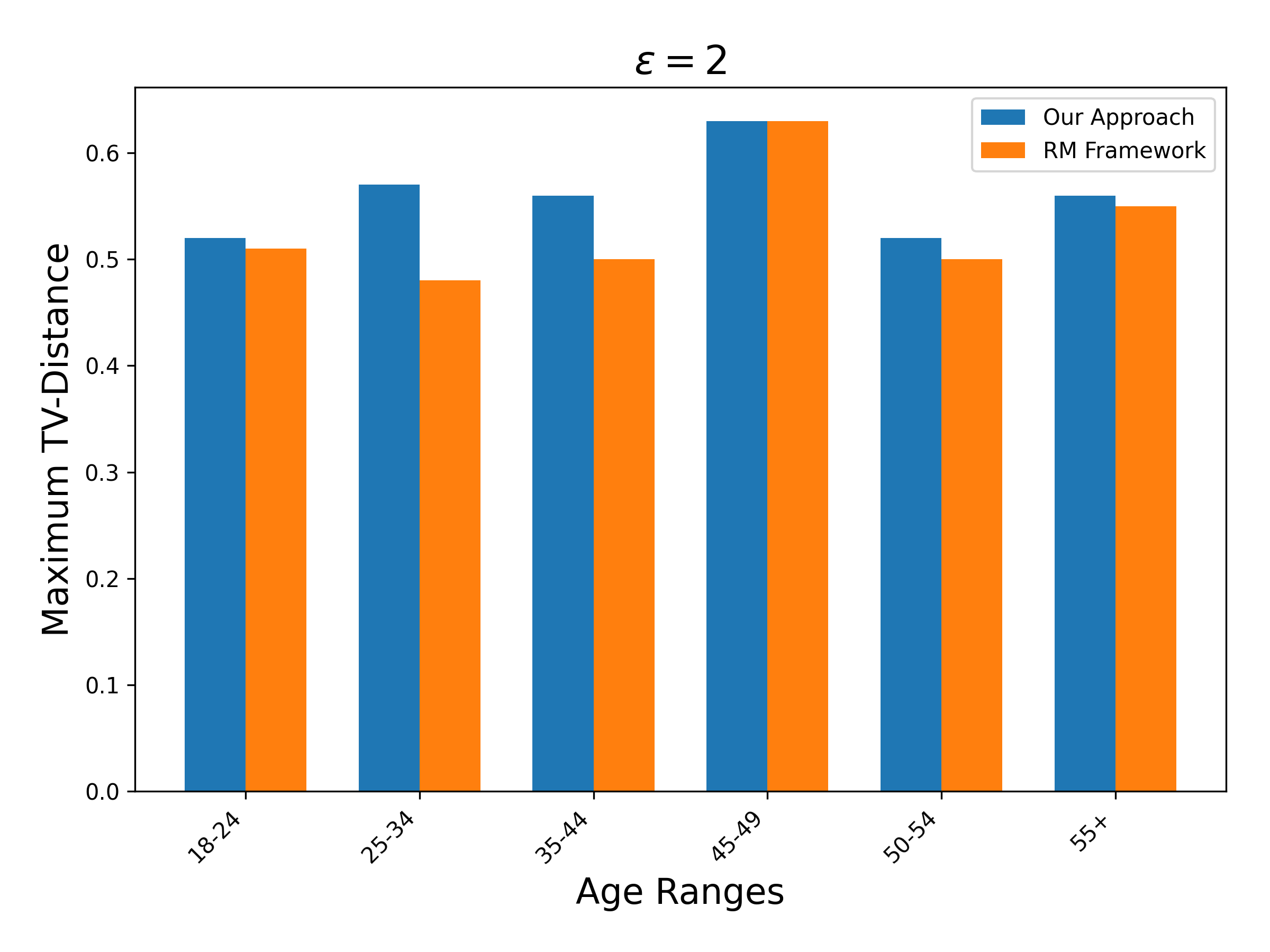}}
    
    \subfloat{\includegraphics[width=0.4\linewidth]{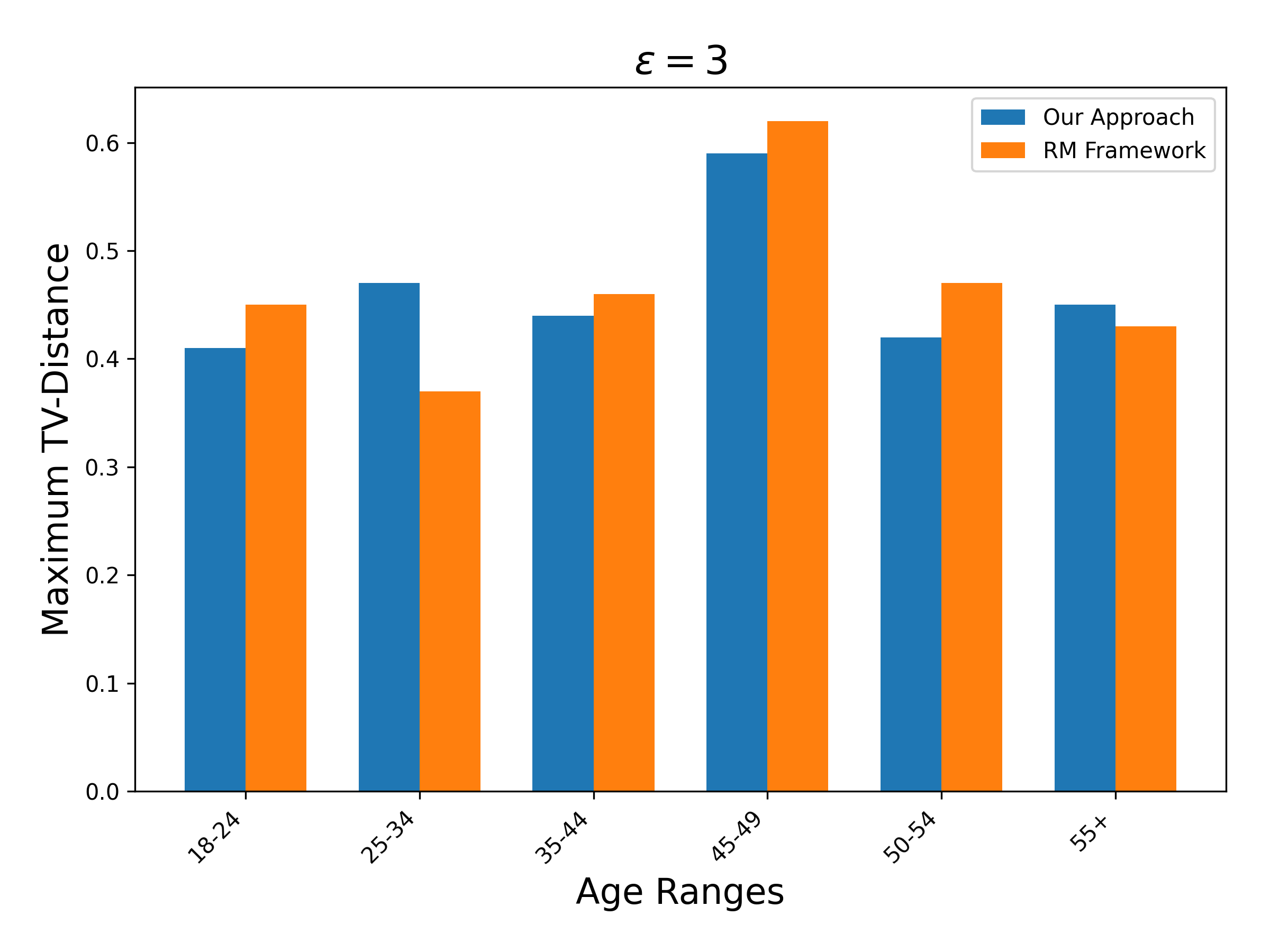}}
    \subfloat{\includegraphics[width=0.4\linewidth]{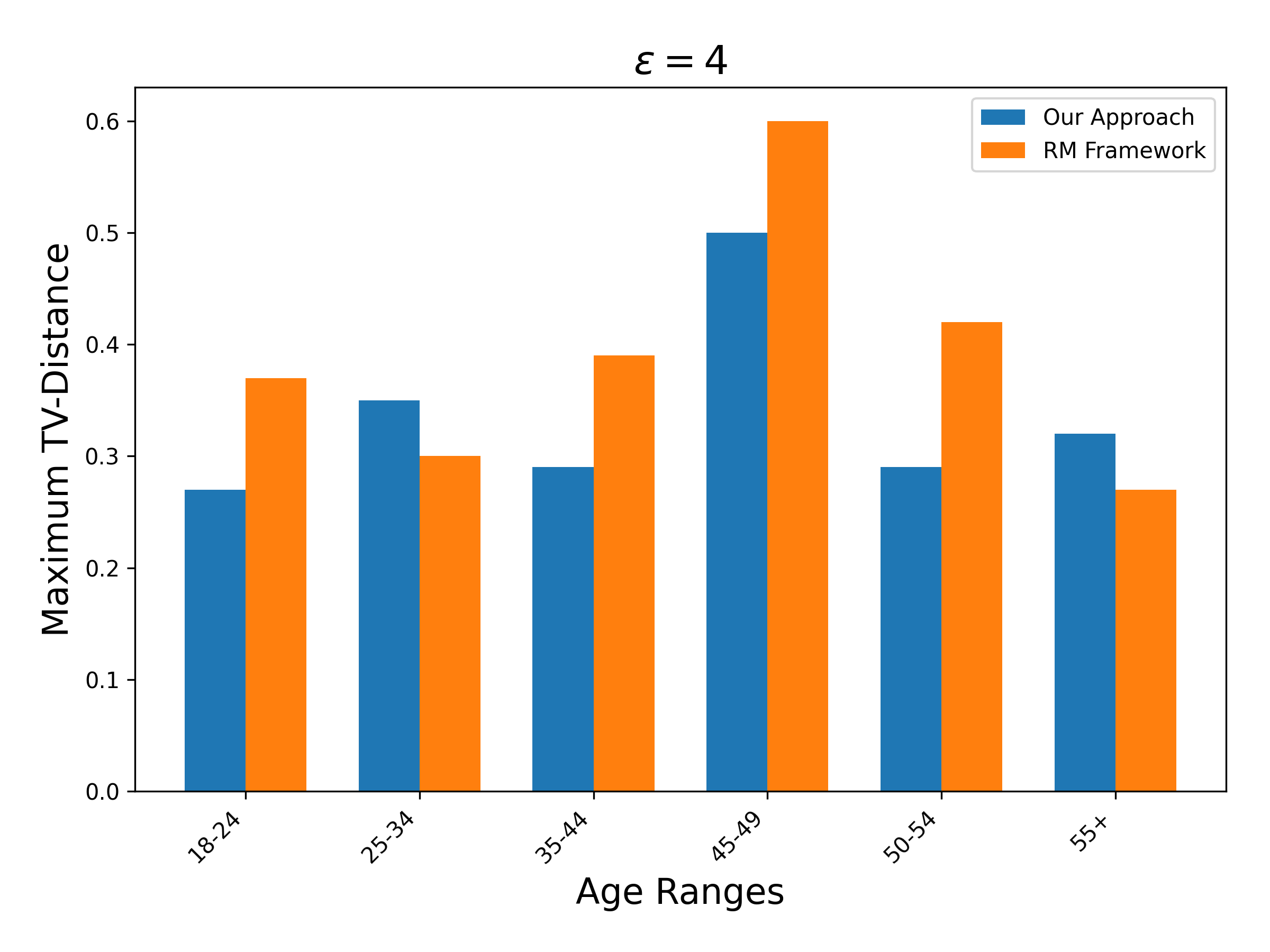}}

    \subfloat{\includegraphics[width=0.4\linewidth]{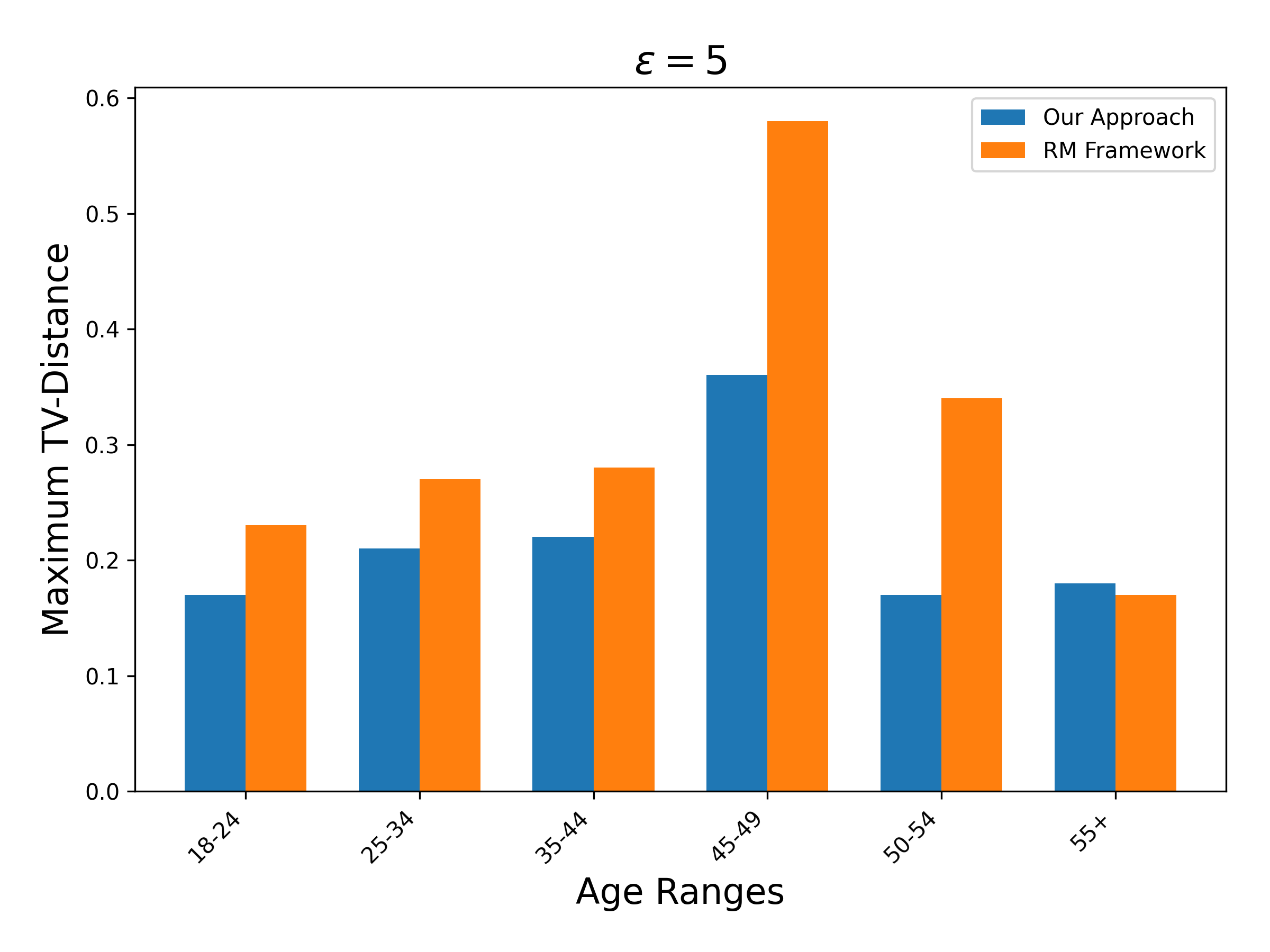}}
    \subfloat{\includegraphics[width=0.4\linewidth]{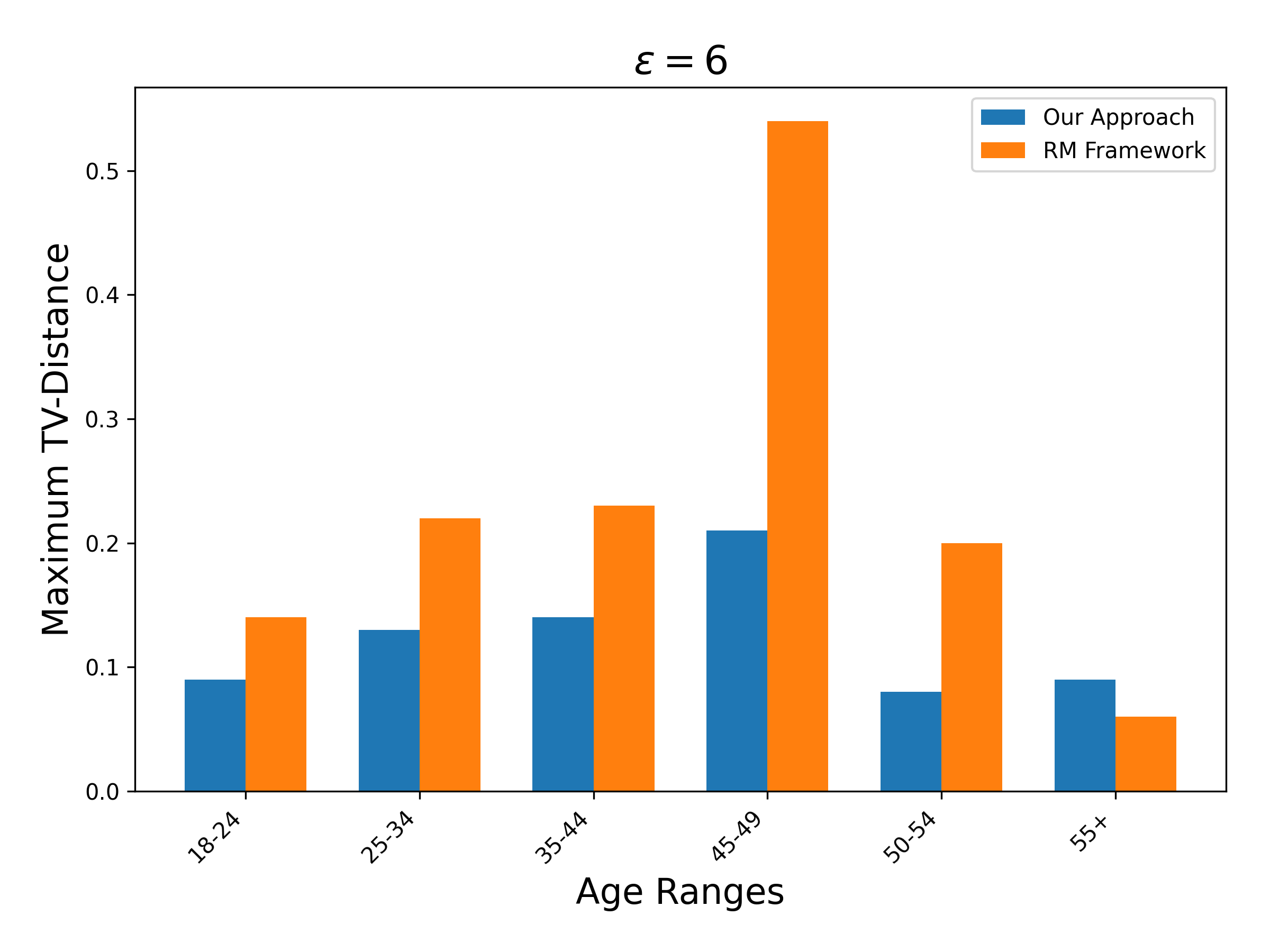}}

    \subfloat{\includegraphics[width=0.4\linewidth]{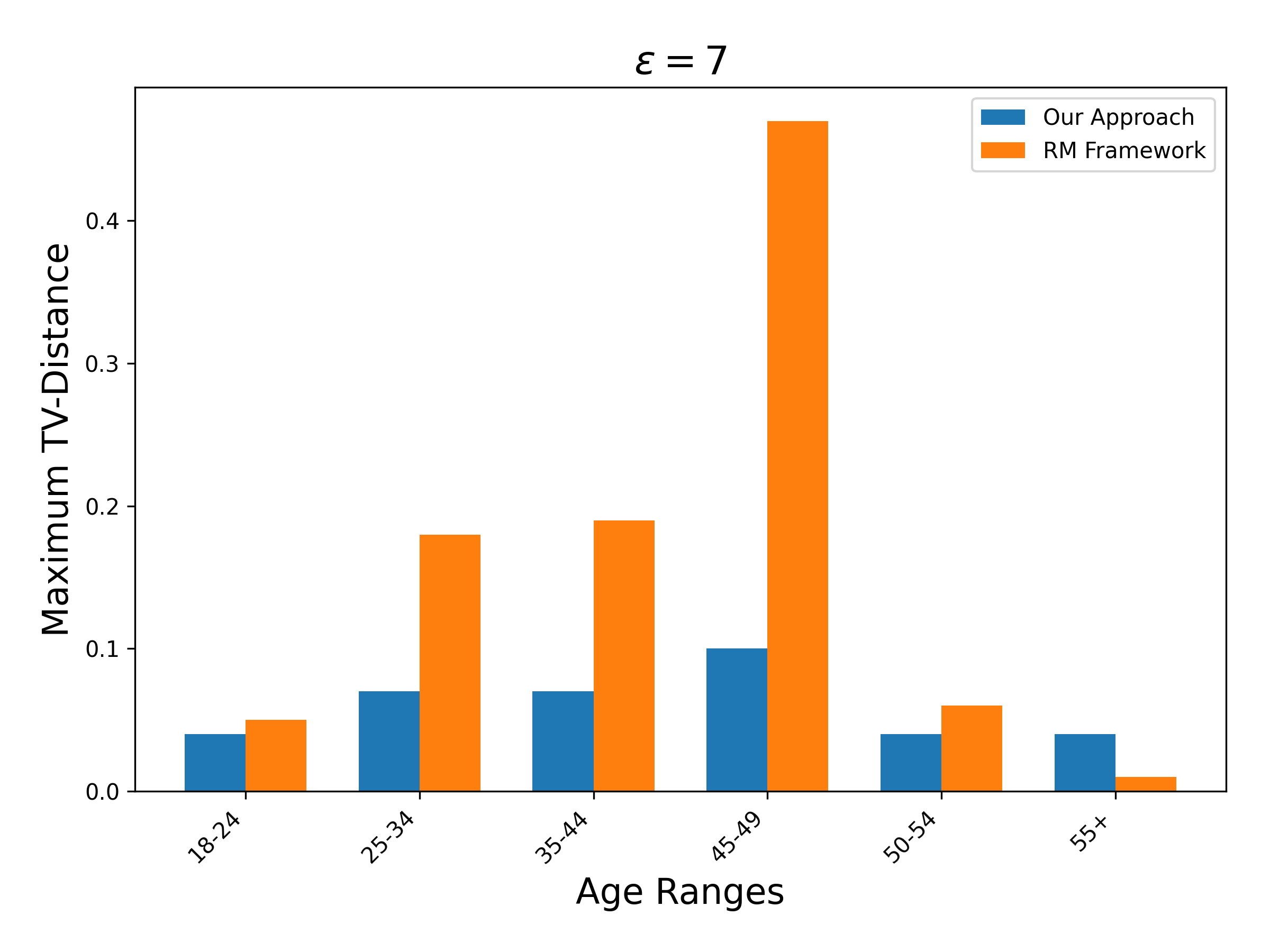}}
    \subfloat{\includegraphics[width=0.4\linewidth]{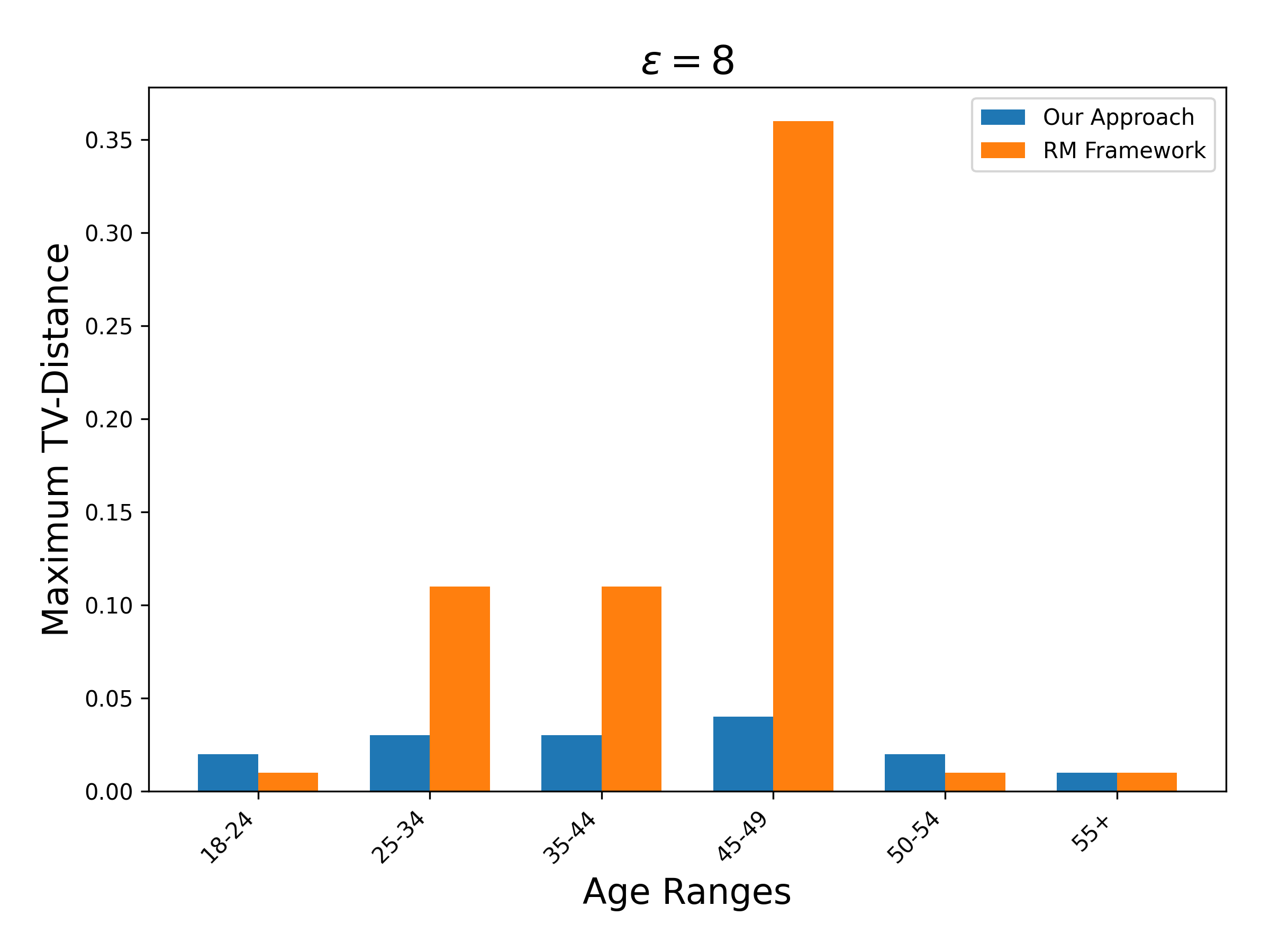}}
    
    \caption{Comparison of our private sampling method with the relative mollifier sampling framework for predicting the genre of the next movie users are likely to watch. The figure reports the maximum $\tv$-distance between users' local distributions and their corresponding sampling distributions across different age groups.}
\end{figure}

\subsection{Computing Infrastructure}\label{ap_infrastructure}

We used Google Colab, which operates on Google Cloud for our computing infrastructure. The CPU is an Intel(R) Xeon(R) CPU @ 2.20GHz with 8 virtual processors and a cache size of 56,320 KB. The system had 53.47 GB of total RAM. The filesystem provided 226GB of storage.

\end{document}